\documentclass[12pt]{colt2023}
\usepackage{amsmath,amssymb}
\usepackage[utf8]{inputenc} 
\usepackage[T1]{fontenc}    
\usepackage{hyperref}       
\usepackage{url}            
\usepackage{booktabs}       
\usepackage{amsfonts}       
\usepackage{nicefrac}       
\usepackage{microtype}      
\usepackage{xcolor}         

\usepackage[utf8]{inputenc}

\usepackage{graphicx}
\usepackage{xcolor}

\usepackage{verbatim,hyperref}
\usepackage{dsfont}
\usepackage{algorithm}
\usepackage[noend]{algpseudocode}
\usepackage{thmtools,thm-restate}

\usepackage{mathtools}
\usepackage{enumerate}
\usepackage{dirtytalk}

\usepackage{stmaryrd}

\usepackage[parfill]{parskip} 
\usepackage{marginnote} 

\usepackage{algorithm}
\usepackage[noend]{algpseudocode}

\usepackage{booktabs}
\usepackage{siunitx}
\usepackage{bbm}

\newcommand{\N}{\mathbb{N}}

\newcommand{\R}{\mathbb{R}}

\newtheorem{fact}{Fact}

\newtheorem{note}{Note}

\newcommand{\e}{\epsilon}

\newcommand{\supp}{\text{supp}}

\newcommand{\samp}{{\mathcal S}}
\newcommand{\dist}{{\mathcal D}}

\newcommand{\loss}{{\ell}}
\newcommand{\expectation}{{\mathbb E}}

\newcommand{\w}{\pmb{w}} 
\newcommand{\ths}{\pmb{t}} 
\newcommand{\slopes}{\pmb{v}} 
\newcommand{\bias}{\pmb{b}} 

\newcommand{\variation}{V}
\newcommand{\vol}{\mathrm{vol}}
\newcommand{\codim}{\mathrm{codim}}

\newcommand{\pred}{{f}}

\newcommand{\noise}{{\eta}}

\newcommand{\prob}{{\mathbb P}}
\newcommand{\reals}{{\mathbb R}}

\newcommand{\ind}{\mathbbm{1}}

\newcommand{\muv}{\mathbf \mu}
\newcommand{\wv}{\mathbf w}
\newcommand{\Wv}{\mathbf W}

\newcommand{\Kv}{\mathbf K}
\newcommand{\model}{\mathcal H}

\newcommand{\famdim}{k}

\newcommand{\inspace}{{\mathcal X}}
\newcommand{\outspace}{{\mathcal Y}}
\newcommand{\sampspace}{{\mathcal Z}}
\newcommand{\btheta}{\boldsymbol{\theta}}

\DeclareMathOperator*{\argmin}{argmin}
\DeclareMathOperator*{\argmax}{argmax}
\usepackage{comment}

\newcounter{mycomment}
\newcommand{\comm}[2]{
\refstepcounter{mycomment}
{%
    \todo[author = \textbf{#1~\#~\themycomment}, color={red!100!green!35}, fancyline, size = \footnotesize]{%
        #2}%
    }
}

\newcommand{\com}[1]{\textcolor{magenta}{gg: #1}}

\usepackage{array,collcell}
\newcommand\AddLabel[1]{%
  \refstepcounter{equation}
  (\theequation)
  \label{#1}
}
\newcolumntype{M}{>{\hfil$\displaystyle}X<{$\hfil}} 
\newcolumntype{L}{>{\collectcell\AddLabel}r<{\endcollectcell}}

\title[Bayes Complexity of Learners vs Overfitting]{Bayes Complexity of Learners vs Overfitting}

\coltauthor{\Name{Grzegorz Głuch} \Email{grzegorz.gluch@epfl.ch}\\
  \Name{Ruediger Urbanke} \Email{ruediger.urbanke@epfl.ch}\\
  \addr EPFL, Lausanne, Switzerland}

\begin{document}

\maketitle

\begin{abstract}

We introduce a new notion of complexity
of functions and we show that it has the following properties: (i) it governs a PAC Bayes-like generalization bound, (ii) for neural networks it relates to natural notions of complexity of functions (such as the variation), and (iii) it explains the generalization gap between neural networks
and linear schemes. While there is a large set of papers which describes bounds that 
have each such property in isolation, and even some that have two, as far as we know, this is a first notion that satisfies all three of them. Moreover, in contrast to previous works, our notion naturally generalizes to neural networks with several layers. 

Even though the computation of our complexity is nontrivial in general, an upper-bound is often easy to derive, even for higher number of layers and functions with structure, such as period functions. 
An upper-bound we derive allows to show a separation in the number of samples needed for good generalization between 2 and 4-layer neural networks for periodic functions. 


\end{abstract}


\section{Introduction}\label{sec:intro}
There is a large body of literature devoted to the question of generalization, both from a practical point of view as well as concerning our theoretical understanding, e.g., \citet{pmlr-v80-arora18b,DBLP:journals/corr/NeyshaburBMS17,DBLP:journals/corr/NeyshaburBMS17aa,NIPS2017_b22b257a,DBLP:journals/corr/NeyshaburTS15, NEURIPS2019_05e97c20}, to mention just a few. We add to this discussion. In particular, we ask what role is played by the hypothesis class assuming a Bayesian point of view. Our main observation is that there is a striking theoretical difference between linear schemes and neural networks. In a nutshell, neural networks, when trained with appropriate gradient methods using a modest amount of training data, strongly prefer hypothesis that are ``easy'' to represent in the sense that there is a large parameter space that approximately represents this hypothesis. For linear schemes no such preference exists. This leads us to a notion of a complexity of a function with respect to a given hypothesis class and prior. We then show that (i) this complexity is the main component in a standard PAC-Bayes bound, and (ii) that the ordering implied by this complexity corresponds well to ``natural'' notions of complexity of functions that have previously been discussed in the literature. In words, neural networks learn ``simple'' functions and hence do not tend to overfit.

For $n \in \N$ we define $[n] = \{1,\dots,n\}$. Let $\inspace$ be the input space, $\outspace$ be the output space and $\sampspace := \inspace \times \outspace$ be the sample space. Let $\mathcal{H}_{\theta}$ be the hypothesis class, parameterized by $\theta \in \mathbb{R}^m $. We define the loss as a function $\ell: \mathcal {H} \times  \sampspace \rightarrow \reals_+$. We focus on the clipped to $C$ version of the quadratic loss but our results can be generalized to other loss functions. We denote by $\dist_x$ a distribution on the input space $\inspace$, and by $\dist$ a distribution on the sample space $\sampspace$. Finally, we let $\samp =\{z_1, \cdots z_N\}$ be the given sample set, where we assume that the individual samples are chosen iid according to the distribution ${\mathcal D}$.

\subsection{The PAC Bayes Bound}
Our starting point is a version of the well-known PAC-Bayes bound, see \citet{pacbayes}.
\begin{lemma}[PAC-Bayes Bound]
Let the loss function $\ell$ be bounded, i.e., $\ell: {\mathcal H} \times {\mathcal Z} \rightarrow [0, C]$. 
Let $P$ be a prior on ${\mathcal H}$ and $Q$ be any other distribution on ${\mathcal H}$ (possibly dependent on $\samp$). Then 
\begin{align} \label{equ:pacbound}
\expectation_{\samp}\left[L_{{\mathcal D}}(Q)\right]  \leq \expectation_{\samp} \left[ L_{\samp}(Q) + C\sqrt{\frac{D(Q \| P)}{2N}} \right],
\end{align}
where
\begin{align} \label{equ:truepacloss}
L_{{\mathcal D}}(Q) & = \expectation_{z \sim {\mathcal D}; h \sim Q}[\ell(h, z)],
\end{align}
\begin{align} \label{equ:empiricalpacloss}
L_{\samp}(Q) & = \expectation_{ h \sim Q}\left[\frac1N\sum_{n=1}^{N} \ell(h, z_i)\right],
\end{align}
and the divergence $D(Q \| P)$ is defined as
\begin{align*}
    D(Q \|P) = \int Q \log \frac{Q}{P}.
\end{align*}
\end{lemma}
There is a large body of literature that discusses use cases, interpretations, and extensions of this bound. Let us just mention a few closely related works.

A related prior notion is that of flat minima. These are minimizers in the parameter space that are surrounded by many functions with similarly small empirical error. The reason for this connection is straightforward. In order for $Q$ to give a good bound two properties have to be fullfilled: (i) $Q$ must be fairly broad so that $D(Q \| P)$ is not too large (afterall, $P$ must be broad since we do not know the function a priori), and (ii) $Q$ must give rise to a low expected empirical error. These properties are exactly the characteristics one expects from a flat minimum. The importance of such minima was recognized early on, see e.g., \citet{hintonflatminima} and \citet{schmidhuberflatminima}. More recently \citet{zecchina1} and \citet{zecchina2} derive from this insight an algorithm for training discrete neural networks that explicitly drives the local search towards non-isolated solution. Using a Bayesian approach they argue that these minima have good generalization. Building on these ideas \citet{dziugaite} give an algorithm with the aim to directly optimize \eqref{equ:pacbound}. They demonstrate empirically that the distributions $Q$'s they find give non-vacuous generalization bounds. 

To summarize, the bound \eqref{equ:pacbound} can be used in various ways. In the simplest case, given a prior $P$ and an algorithm that produces a ``posterior'' $Q$, (\ref{equ:pacbound}) gives a probabilistic upper bound on the average true risk if we sample the hypothesis according to $Q$. But (\ref{equ:pacbound}) can also be taken as the starting point of an optimization problem. Given a prior distribution $P$ one can in principle look for the posterior $Q$ that gives the best such bound. Further, one can split the available data and use one part to find a suitable prior $P$ and the remaining part to define a that posterior $Q$ distribution that minimizes this bound.  

We take the PAC-Bayes bound as our starting point. We impose a Gaussian distribution on the weights of the model. This defines our prior $P$. In principle other priors can be used for our approach but a Gaussian is the most natural choice and it illustrates the main point of the paper in the cleanest fashion. Further, we postulate that the samples $z_n=(x_n, y_n)$ are iid and, assuming that the true parameter is $\theta$, come from the stochastic model

\begin{align} \label{equ:datamodel}
x_n \mapsto y_n = f_{\theta}(x_n) + \eta_n, \eta_n \sim \mathcal{N}(0,\sigma_e^2).
\end{align}
In words, we assume that the actual underlying function is {\em realizable}, that we receive noisy samples, and that the noise is Gaussian and independent from sample to sample.  

This gives rise to the posterior distribution,
\begin{align} \label{equ:posterior}
Q(\theta) = \frac{P(\theta) e^{- \frac{1}{2 \sigma_y^2} \sum_{n=1}^{N} (y_n - f_{\theta}(x_n))^2}}{\int P(\theta') e^{- \frac{1}{2 \sigma_y^2} \sum_{n=1}^{N} (y_n - f_{\theta'}(x_n))^2} d \theta'}.
\end{align}
One valid criticism of this approach is that it is model dependent. But there is a significant payoff. First recall that this posterior can at least in principle be sampled by running the SG algorithm with Langevin dynamics. For the convenience of the reader we include in Section~\ref{sec:langevin} a short review. Most importantly, taking this point of view a fairly clear picture arises why neural networks tend not to overfit. In a nutshell, if we sample from this posterior distribution then we are more likely to sample ``simple'' functions. The same framework also shows that this is not the case for linear schemes.
 
\subsection{Stochastic Gradient Langevin Dynamics} \label{sec:langevin}
We follow \citet{marceaucaron2017natural}. Assume that we are given the data set
\begin{align*}
\samp = \{z_1, \cdots, z_N\} = \{(x_1, y_1), \cdots, (x_N, y_N)\},
\end{align*}
where the samples $z_n=(x_n, y_n)$, $n=1, \cdots, N$, are chosen iid according to an unknown distribution $\dist$. We model the relationship between $x$ and $y$ probabilistically in the parametrized form
\begin{align*}
y \sim p(y \mid x, \theta).
\end{align*}
We use the log-loss 
\begin{align*}
\loss_{\theta}(x, y) = - \ln p(y \mid x, \theta).
\end{align*}
Assume further that we use the {\em stochastic gradient Langevin descent} (SGLD) algorithm:
\begin{align*}
  \theta^{(t)} &= \theta^{(t-1)} - \eta \expectation_{Z \sim \dist} \left[\nabla_{\theta} \loss_{\theta}(X, Y) - \frac1N \ln P(\theta) \right] \\
  &+ \sqrt{\frac{2 \eta}{N}} {\mathcal N}(0, I),
\end{align*}
where $t = 1, 2, \cdots$;  $\eta>0$ is the learning rate, $P(\theta)$ is the density of the prior, and ${\mathcal N}(0, I)$ denotes a zero-mean Gaussian vector of dimension $\text{dim}(\theta)$ with iid components and variance $1$ in each component.

Note that due to the injected noise, the distribution of $\theta$ at time $\tau$, call it $\pi_{\tau}(\theta)$, converges to the posterior distribution of $\theta$ given the data, i.e., it converges to 
\begin{align*}
&p(\theta \mid \{z_1, \cdots, z_N\}) 
= \frac{p(\theta, \{z_1, \cdots, z_N\})}{p(\{z_1, \cdots, z_N\})} \\
&=
\frac{P(\theta) p(\{z_1, \cdots, z_N\} \mid \theta)}{p(\{z_1, \cdots, z_N\})} \\
 & =\frac{P(\theta) \prod_{n=1}^{N} p(y_n \mid x_n, \theta)}{\prod_{n=1}^{N} p(y_n \mid x_n)} \propto P(\theta) \prod_{n=1}^{N} p(y_n \mid x_n, \theta).
\end{align*}
This is shown in \citet{teh2015consistency,chen2016convergence}. In the sequel we use the more common notation $p_{\theta}(y_n \mid x_n)$ instead of $p(y_n \mid x_n, \theta)$. This makes a clear distinction between the parameters of the model and the samples we received.

A few remarks are in order. An obvious choice from a theoretical point of view is to use an iid Gaussian prior. In practice it is best not to use iid Gaussian prior in order to speed up the convergence. Indeed, the main point of \citet{marceaucaron2017natural} is to discuss suitable schemes. But for our current conceptual purpose we will ignore this (important) practical consideration. 

\section{The PAC Bayes Bound and Bayes Complexity}
Let us now get back to the main point of this paper. We start by defining two notions of complexity. Both of them are ``Bayes'' complexities in the sense that both relate to the size of the parameter space (as measured by a prior) that approximately represents a given function. We will then see how this complexity enters the PAC-Bayes bound.

\paragraph{Contribution.} Our main contribution is an introduction of a new notion of complexity of functions and we show that it has the following properties: (i) it governs a PAC Bayes-like generalization bound, (ii) for neural networks it relates to natural notions of complexity of functions, and (iii) it explains the generalization gap between neural networks and linear schemes in some regime. While there is a large set of papers which describes each such criterion, and even some that fulfill both (e.g., \citet{srebronormbound}), as far as we know, this is a first notion that satisfies all three of them. 

\begin{definition}[Sharp complexity]\label{def:sharpcomplexity}
For every $\e > 0$ we define the sharp complexity of a function $g$ with respect to the hypothesis class $\mathcal{H}_\theta$ as
\begin{align*}
\chi^\#(\mathcal{H}_\theta, g, \dist_x, \e^2) &:= -\log \left[ \prob_\theta\{ \theta : \expectation_{x \sim \dist_x} [(g(x) - f_{\theta}(x))^2] \leq \e^2 \} \right],   
\end{align*}
where the probability $\prob_\theta$ is taken wrt to the prior $P$.
\end{definition}

In words, we compute the probability, under prior $P$, of all these functions $f_\theta$ that are close to $g$ under the quadratic loss and distribution $\dist_x$.

In general, it is difficult to compute $\chi^\#$ for a given
$\epsilon$.  However, for realizable functions it is often possible to compute
the limiting value of the sharp complexity, properly normalized, when $\epsilon$
tends to $0$.
\begin{definition}[Limiting Complexity]\label{def:limitingcomp}
We define the sharp complexity of a function $g$ with respect to the hypothesis class
\begin{equation} \label{comp-def}
\chi^\#(\mathcal{H}_\theta, g, \dist_x)     := \lim_{\epsilon \rightarrow 0}  \frac{\log \left[\prob_\theta \left\{\theta: \expectation_{x \sim \dist_x} \left[ (g(x) - f_\theta(x))^2 \right] \leq \e^2 \right\}\right]}{\log(\e)}.
\end{equation}
\end{definition}

The above definitions of complexity implicitly depend on the hypothesis class $\mathcal{H}_\theta$. If the hypothesis class (and/or $\dist_x$) is clear from context we will omit it from notation, e.g. $\chi^\#(g, \e^2) = \chi^\#(g, \dist_x, \e^2) = \chi^\#(\mathcal{H}_\theta, g, \dist_x, \e^2)$. 

We now state the main theorem. It is a generalization bound, which crucially depends on the sharp complexity from Definition~\ref{def:sharpcomplexity}. The proof is deferred to Appendix~\ref{apx:generalization}.

\begin{theorem}\label{thm:generalizationbound}
If $L_\dist(P) \geq 2\sigma_e^2$ and $g \in \text{supp}(P)$ then for every $ \beta \in (0,1]$ there exists $\sigma_{\text{alg}}^2$ such that if we set $\sigma_y^2 = \sigma_{\text{alg}}^2$ then $\expectation_{\samp \sim \dist^N}[L_S(Q(\sigma_y^2))] = (1+\beta)\sigma_e^2$ and
\begin{align}
\expectation_{\samp \sim \dist^N}[L_\dist(Q(\sigma_y^2))]
&\leq
\sigma_e^2 + \left[ \beta \sigma_e^2 + \frac{C}{\sqrt{2}}\sqrt{\frac{\chi^\#(g, \dist_x, \beta \sigma_e^2)}{N}} \right]. \label{eq:mainthm}
\end{align}
\end{theorem}

\paragraph{Discussion of Assumptions.}
Requiring that $g \in \text{supp}(P)$ is only natural as it indicates that $g$ is realizable with prior $P$. It is also natural to assume that $L_\dist(P) \geq 2\sigma_e^2$ as the lowest possible error is attained by $g$ and is equal $\sigma_e^2$. Thus we require that the expected loss over the prior is twice as big as the minimal one. As $P$ should cover a general class of functions it is only natural that $L_\dist(P) \geq 2\sigma_e^2$.

For a fixed $\beta$, $\sigma_{\text{alg}}^2$ from Theorem~\ref{thm:generalizationbound} is, in general, not known. However, as proven in Appendix~\ref{apx:generalization}, we have
$$
\lim_{\sigma_y^2 \rightarrow 0} \expectation_{\samp \sim \dist^N}[L_\samp(Q(\sigma_y^2))]= \sigma_e^2, \ \ \lim_{\sigma_y^2 \rightarrow \infty} \expectation_{\samp \sim \dist^N}[L_\samp(Q(\sigma_y^2))]= 2\sigma_e^2.
$$
Moreover, $\expectation_{\samp \sim \dist^N}[L_\samp(Q(\sigma_y^2))]$ is continuous in $\sigma_y^2$, which implies that $\sigma_{\text{alg}}^2$ can be found by a binary search-like procedure by holding out some part of $\samp$ for estimating $\expectation_{\samp \sim \dist^N}[L_\samp(Q(\sigma_y^2))]$ for different $\sigma_y^2$ values.

\paragraph{Bound \eqref{eq:mainthm} in terms of limiting complexity.}  Notice that \eqref{eq:mainthm} is governed by 
$\chi^\#(g,\dist_x,\beta \sigma_e^2)$. Aaccording to \eqref{comp-def}, for small enough $\beta\sigma_e^2$, we have
\begin{equation}\label{eq:relationbetweencomp}
\chi^\#(g,\dist_x,\beta \sigma_e^2) \approx -\chi^\#(g,\dist_x)\log(\beta \sigma_e^2).
\end{equation}
This means that for small enough noise level, where the exact regime for which the approximation holds depends on a specific problem, we have
\begin{equation}\label{eq:generalizationlimit}
\expectation_{\samp \sim \dist^N}[L_\dist(Q(\sigma_y^2))]
\lessapprox
(1 + \beta) \sigma_e^2 + \frac{C}{\sqrt{2}}\sqrt{\frac{-\chi^\#(g, \dist_x)\log(\beta \sigma_e^2)}{N} }.    
\end{equation}
We see that the generalization bound depends crucially on the limiting complexity. 

\paragraph{Main message.} 
Note that the smallest we can hope to get on the right hand side is $\sigma_e^2$ since this is the variance of the noise and this is achievable if we pick $Q$ that puts all its mass on $g$.
This means that $\beta \sigma_e^2$ plus the square root term from \eqref{eq:mainthm} represents the expected excess generalization error. 

This brings us to the punch line of this paper. In the subsequent sections we will see that
(i) natural notions of complexity that have previously been discussed in the literature align with our new notion when we consider neural networks, whereas 
(ii) for linear schemes our notion of complexity is essentially independent of the function (as long as it is realizable) and as a consequence is as high as for the most complex (in the natural sense) function in our hypothesis class.

To the degree that we assume that reality prefers simple functions this explains why neural nets generalize significantly better than linear schemes.

In Section~\ref{sec:complexityGaussianCase} we show that for neural networks and a piece-wise linear function $g$ the limiting complexity is equal to the number of slope changes $g$. In light of \eqref{eq:generalizationlimit}, this means that neural networks require the fewer samples (for a good generalization bound) the fewer slope changes $g$ has.

There is a further connection to a natural notion of complexity. In Section~\ref{sec:epscompcase} we show that sharp complexity is related to the variation of $g$, i.e. the integral of the second derivative of $g$. Thus, in the light of \eqref{eq:mainthm}, fewer samples are needed (for a good generalization) for $g$'s with smaller variation.

As we discussed above, sharp and limiting complexity are related via \eqref{eq:relationbetweencomp} when $\beta \sigma_e^2$ is small. We can thus think of sharp complexity as a refinement of limiting complexity. This is reflected in the two cases discussed above -- the number of slope changes can be seen as an \say{approximation} of the variation of a function.

In Section~\ref{sec:linear}, on the other hand, we show that for linear schemes the limiting complexity is virtually independent of the function and \say{equal} to the number of basis functions. This means that in this case the number of samples needed for a good generalization bound is the same for simple and complicated functions.



\section{Models}

Although the basic idea applies to any parametric family, we will consider restricted types of families and demonstrate our concepts with two concrete examples, namely linear schemes and NNs. We will be interested in parametric families of functions from $\inspace$ to $\outspace$. More precisely families of the form $\mathcal{H}_{\theta} := \left\{\pred_{\theta} : \inspace \xrightarrow{} \outspace, \theta \in \R^m \right\}$,
where $\theta$ is the vector of parameters. for a function $g : \inspace \rightarrow \outspace$ and a distribution $\dist_x$ we define the set of exact representations as $A_{g,\mathcal{H},\dist_x} := \{\theta \in \R^m : f_\theta \equiv_{\supp(\dist_x)} g \}$. If $\mathcal{H}$ and $\dist_x$ are clear from context we will often write $A_g$. The $0$ function will play an important role, thus we also define $A_0 := \{\theta \in \R^m : f_\theta \equiv_{\supp(\dist_x)} 0 \}$

\subsection{Linear Schemes} \label{sec:linearmodel}
Consider the linear family $\model_\theta^{(\text{L, o})}=\{f_{\theta}(x): f_{\theta}(x) = \sum_{i=0}^{d-1} \w_i b_i(x), x \in \inspace = [-1, 1]\}$,
i.e., the vector of parameters $\theta$ is equal to the vector of weights $\w$. We assume that the functions $\{b_i(x)\}$ form an orthonormal basis.
Although the exact basis that is used is not of importance one might think of $b_i(x)$ as a polynomial of degree $i$ or the first few Legendre polynomials. In this way the basis functions are naturally ordered by complexity.

\subsection{Neural Networks} \label{sec:nn}
Consider the family $\model^{\text{NN}}$ represented by NNs with layers numbered from $0$ (input) to $K$ (output), containing $d = d_0, d_1, \dots$, and $d_K = d_y$ neurons respectively. We will limit our attention to $d_y = 1$. The activation functions for the layers $1$ to $K$ are presumed to be $\sigma_1, \dots, \sigma_K : \R \xrightarrow{} \R$. The weight matrices will be denoted by $W^{(1)}, W^{(2)}, \dots, W^{(K)}$, respectively, where matrix  $W^{(k)}$ connects layer $k-1$ to layer $k$. We define 
\begin{equation}\label{eq:defoff}
\pred_{\theta}(x) := \sigma_K (\bias^{(K)} + W^{(K)} \sigma_{K-1}( \dots \sigma_{1}(\bias^{(1)} + W^{(1)} x ))) .
\end{equation}

\section{Why Neural Nets Generalize Well}\label{sec:nngeneralizewell}
We now get to the main point of this paper, namely why neural nets generalize much better than other schemes, in particular linear schemes.

The basic idea is simple. We have seen in the previous sections that (i) a suitable version of SGD gives us a posterior of the form (\ref{equ:posterior}), and (ii) this posterior gives rise to a an upper bound on the generalization error that depends mainly on the ``complexity'' of the underlying true hypothesis.

This notion of complexity of a function depends on the underlying hypothesis class.
To close the circle we will now discuss how this complexity behaves for interesting hypothesis classes. In particular, as we will see that there is a striking difference between linear schemes and neural networks. For linear schemes, every realizable function has essentially the same complexity. This in particular means that we do not expect to learn a ``simple'' function (e.g., think of a constant function) with fewer samples than a ``complex'' one (think of a highly variable one). For neural nets the complexity behaves entirely differently and there is a large dynamic range. As we will see, in a suitable limit the complexity is to first order determined by the number of degrees of freedom that have to be fixed in order to realize a function. Therefore, for neural nets, simple functions have a much lower complexity than complicated ones. 

\subsection{Neural Networks with a Single Hidden Layer}

We start with analyzing our notion of complexity for the case of NN with a single hidden layer and $1$-dimensional input. More precisely let $x \in \R$ denote the input and $y \in \R$ denote the output.  There are $k$ nodes in the hidden layer. More precisely, the network represents the function
 \begin{align*}
f_{\theta}(x) 
& = \sum_{i=1}^k \w_i^{(2)} \sigma \left( \w_i^{(1)} x + \bias_i^{(1)}\right)  + b^{(2)} \\
& = \sum_{i=1}^k \w_i^{(2)} \left[ \w_i^{(1)} x + \bias_i^{(1)} \right]_+ + b^{(2)},
\end{align*}
i.e., we use ReLU activation functions. 
The $\bias_i^{(1)}$ denotes the bias of the $i$-th node, the $\w_i^{(2)}$ represents the weight of the $i$-th output signal, and $b^{(2)}$ is the global bias term of the output. We let $\theta = (\theta_w, \theta_b) = ((\w^{(1)}, \w^{(2)}), (\bias^{(1)}, b^{(2))})$ denote the set of all parameters, where $\theta_w$ denotes the set of weights and $\theta_b$ denotes the set of bias terms.

\paragraph{Parametrization and prior.} We will use the following non-standard parametrization of the network
\begin{align*}
f_\theta(x) 
&= \sum_{i=1}^k \w_i^{(2)} \left[\w_i^{(1)}\left(x - \bias_i^{(1)}\right) \right]_+ + b^{(2)} \\
&= \sum_{i=1}^k \w_i^{(2)} \cdot \left|\w_i^{(1)}\right| \cdot \left[\text{sgn}(\w_i^{(1)})( x - \bias_i^{(1)})\right]_+ + b^{(2)},
\end{align*}
where in the last equality we used the fact that the ReLU activation function is 1-homogenous. Note that there are two kinds of ReLU functions (depending on the sign of $w_i^{(1)}$) they are either of the form $[x -b]_+$ or $0$ at $[-(x-b)]_+$. If we restrict our attention to how $f_\theta$ behaves on a compact interval then considering just one of the kinds gives us the same expressive power as having both. This is why for the rest of this section we restrict our attention only to the case of $[x-b]_+$ as it simplifies the proofs considerably. Thus the final parametrization we consider is
$$
f_\theta(x) = \sum_{i=1}^k \w_i^{(2)} \cdot \w_i^{(1)}\left[x - \bias_i^{(1)}\right]_+ + b^{(2)}.
$$

We define the prior on $\theta$ as follows: each component of $\theta_w$ comes i.i.d. from $\mathcal{N}(0,\sigma_y^2)$, each component of $\bias^{(1)}$ comes i.i.d. from $U([0,M])$, where $M$ will be fixed later and $b^{(2)}$ comes i.i.d. from $\mathcal{N}(0,\sigma_b^2)$\footnote{The different parametrization and the uniform prior on the bias terms are non-standard choices that we make to simplify the proofs. These choices would not affect the spirit of our results but as always the details need to be verified.}.

We will argue that our notion of complexity ($\chi^\#(\dist_x,g,\e^2)$ 
and $\chi^\#(\dist_x,g)$) corresponds, in a case of NN, to natural notions of complexity of functions. 


\paragraph{Target function.} We will be interested in target functions $g$ that are representable with a single hidden layer networks. Let $g : [0, 1] \xrightarrow{} \R$ be continuous and piece-wise linear. I.e., there is a sequence of points $0=t_1 < t_2 < \cdots < t_{l+1}=1$ so that for $x \in [t_i, t_{i+1}]$, $1 \leq i < l+1$,
\begin{align} \label{equ:polygone}
g(x) = c_i + \alpha_i (x-t_i),
\end{align}
for some constants $c_i$ and $\alpha_i$, where $c_{i+1} = c_i + \alpha_i (x_{i+1}-x_i)$. Then $f$ can be written as a sum of ReLU functions, 
\begin{align} \label{equ:firstrepresentation}
g(x) = b + \sum_{i=1}^{l} v_i [x-t_i]_+,
\end{align}
where $v_1=\alpha_1$ and $v_{i}=\alpha_{i}-\alpha_{i-1}$, $i=2, \cdots, l$. The terms in \eqref{equ:firstrepresentation} for which $v_i = 0$ can be dropped without changing the function. We call the number of nonzero $v_i$'s in \eqref{equ:firstrepresentation} to be the number of changes of slope of $g$.

\subsubsection{Complexity in the Asymptotic Case} \label{sec:complexityGaussianCase}
In this section we explore what is the limiting value of the sharp complexity for the case of NN.

It will turn out that the key object useful for computing $\chi^\#(g)$ is a particular notion of dimension of $A_g$. 
\begin{definition}[Minkowski–Bouligand co-dimension]
For $A, S \subseteq \R^m$ we define the Minkowski-Bouligand co-dimension of $A$ w.r.t. $S$ as
$$\codim_S(A) := \lim_{R \rightarrow 
\infty} \lim_{\e \rightarrow 0} \frac{\log(\vol( (A + B_\e) \cap B_R \cap S))}{\log(\e)} , $$
where $\vol$ is the Lebesgue measure and $+$ denotes the Minkowski sum. 
\end{definition}

\begin{note}
Our definition is a variation of the standard Minkowski-Bouligand dimension. The first difference is that we measure the co-dimension instead of the dimension. Secondly,  we compute $\lim_{R \rightarrow \infty}$. We do this because the sets we will be interested in are unbounded. We also define the co-dimension wrt to an auxilary set $S$, i.e., all volumes are computed only inside of $S$. One can view it as restricting the attention to a particular region. In our use cases this region will be equal to the support of the prior. We will sometimes use $\codim_P(A)$ to denote $\codim_{\supp(P)}(A)$, when $P$ is a distribution.

Technically the notion is not well defined for all sets. Formally, one defines a lower and an upper co-dimension, corresponding to taking $\liminf$ and $\limsup$. Sets $A$ and $S$ need also be measurable wrt to the Lebesgue measure. We will however assume that for all of our applications the limits are equal, sets are measurable and thus the co-dimension is well defined. This is the case because all sets we will be interested in are defined by polynomial equations.
\end{note}

The first lemma relates sharp complexity and co-dimension.

\begin{lemma}\label{lem:complexitytodimension}
Let $g(x) = b + \sum_{i=1}^c v_i [x - t_i]_+,$
where $0 < t_1 < \dots < t_c < 1$, $v_1,\dots,v_c \neq 0$ and $ c \leq k$. Then
$$
\frac1{5} \codim_{P}(A_g) \leq \chi^\#(g, U([0,1])) \leq \codim_{P}(A_g).
$$
Recall that $A_g = \{ \theta : f_\theta \equiv_{[0,1]} g\}$.
\end{lemma}

The next lemma computes the co-dimension of a function with $c$ changes of slope.

\begin{lemma}[Function with $c$ changes of slope - co-dimension]\label{lem:cchangesco-dimension}
Let $g(x) = b + \sum_{i=1}^c v_i [x - t_i]_+,$
where $0 < t_1 < \dots < t_c < 1$, $v_1,\dots,v_c \neq 0$ and $ c \leq k$. Then
$$
\codim_P(A_g) = 2c+1.
$$
\end{lemma}

This brings us to the main result of this subsection

\begin{example}[Function with $c$ changes of slope - Bayes Complexity]\label{lem:cchangeslimit}
Let $g : [0,1] \rightarrow \R$ and assume that $g$ is a piece-wise linear function with $c \leq k$ changes of slope. Then
$$
\frac{2c+1}{5} \leq \chi^\#(g, U([0,1])) \leq 2c+1.
$$
\end{example}

We see that the limiting complexity is $\approx c$, for $ c \leq k$. This means that the complexity depends strongly on the function and simpler - in a sense of fewer changes of slopes - functions have smaller complexity. In Section~\ref{sec:linearmodel} we will compute the limiting complexity for linear models. It will turn out, see Example~\ref{ex:linearlimit}, that in this case the complexity doesn't depend on the function and is equal to the number of basis functions used in the linear model.

\subsubsection{The $\e$-Complexity Case}\label{sec:epscompcase}

We saw in the previous section that for the case of neural networks our notion of complexity corresponds (in the limit and up to constant factors) to the number of degrees of freedom that need to be fixed to represent a given function.
When we evaluate the complexity at more refined scales it can be shown that it is closely related to another natural complexity measure. 

\paragraph{Variational Complexity}
Let us now introduce a \say{complexity} measure for a piece-wise linear function $g$. 
We start by introducing a complexity measure for a particular choice of the network parameters. The complexity of the function will then be the minimum complexity of the network that represents this function. 
We choose
\begin{align} \label{equ:complexitymeasure}
C_k(\theta) = \frac12 \| \theta_w\|^2 = \frac12\left( \|\w^{(1)}\|^2 + \|\w^{(2)}\|^2 \right),
\end{align}
i.e., it is half the squared Euclidean norm of the {\em weight parameters}. 

If we use the representation (\ref{equ:firstrepresentation}) in its natural form, i.e.,  $w^{(2)}_i =a_i$ and $W^{(1)}_i = 1$, then we have $C_k(\theta) = \frac12 \sum_{i=1}^{k} (a_i^2+1)$. But we can do better. Write
\begin{align} \label{equ:secondrepresentation}
f(x) = c + \sum_{i=1}^{k} w^{(2)}_i [W^{(1)}_i(x-x_i)]_+,
\end{align}
where $w^{(2)}_i =a_i/\sqrt{|a_i|}$ and $W^{(1)}_i = |w^{(2)}_i |$. This gives us a complexity measure $C_k(\theta) = \sum_{i=1}^{k} |a_i| = \sum_{i=1}^{k} |\alpha_i-\alpha_{i-1}|$, where $\alpha_0=0$. Indeed, it is not very hard to see, and it is proved in \citet{srebronormbound}, that this is the best one can do even if we keep $f(x)$ fixed and are allowed to let the number $k$ of hidden nodes tend to infinity. In other words, for the function $f$ described in (\ref{equ:polygone}) we have
\begin{align*}
C(f) = \inf_{k \in \N, \theta: f_\theta = f} C_k(\theta) = \variation(f'),
\end{align*}
where $\variation(f')$ denotes the total variation of $f'$, the derivative of $f$. Why total variation?
Note that $\alpha_i$ denotes the derivative of the function so that $|\alpha_i-\alpha_{i-1}|$ is the change in the derivative at the point $x_i$. Therefore, $\sum_{i=1}^{k} |\alpha_i-\alpha_{i-1}|$ is the total variation associated to this derivative. 


If we consider a general function $f: [0, 1] \xrightarrow{} \R$ then for every $\epsilon>0$, $f$ can be uniformly approximated by a piecewise linear function, see \citet{shekhtman82}. As $\epsilon$ tends to $0$ for the \say{best} approximation the variation of the piece-wise linear function converges to the total variation of $f'$. This can equivalently be written as the integral of 
$|f''|$.
It is therefore not surprising that if we look at general functions $f: \R \xrightarrow{} \R$ and let the network width tend to infinity then the lowest cost representation has a complexity of
\begin{align} \label{equ:complexity}
C(f) = \max \left(\int |f''(x)| dx, |f'(-\infty) + f'(+\infty)| \right).
\end{align}
As we previously mentioned, this concept of the complexity of a function was introduced in \citet{srebronormbound} and this paper also contains a rigorous proof of (\ref{equ:complexity}). (Note: The second term in \eqref{equ:complexity} is needed
when we go away from a function that is supported on a finite domain to $\R$. To see this consider the complexity of $f(x) = \alpha x$. It is equal to $2\alpha$ ($f(x) = \sqrt{\alpha} [\sqrt{\alpha} x]_+ - \sqrt{\alpha} [-\sqrt{\alpha} x]_+$) but $\int |f''(x)| dx = 0$.)

\paragraph{Sharp versus Variational Complexity.} Now we explain how the notion of sharp complexity is, in some regimes, equivalent to the variational complexity. This gives a concrete example of our promise that sharp complexity aligns well with natural complexity measures.


Assume at first that the target function is of the form $g(x) = b + \sum_{i=1}^c v_i[x - t_i]_+$
and requires only a single change of the derivative. I.e., the piece-wise linear function consists of two pieces and we require only one term in the sum, $g(x) = a[x - t]_+ + b$ Call this function $g_1$, where the $1$ indicates that there is only a single change of the derivative and the change is of magnitude $a$. 

We now ask what is the value of $\chi^\#(g_1, \dist_x, \e^2)$, for $\dist_x = U([0,1])$ - as this is what appears in \eqref{eq:mainthm}. We claim that for small $\e$, specific choices of $M$ and $\sigma_w^2, \sigma_b^2$ and particular regimes of parameters we have
\begin{equation}\label{eq:promise}
\chi^\#(g_1, U([0,1]), \e^2) = \Theta(a / \sigma_w^2) = \Theta(C(g_1) / \sigma_w^2).
\end{equation}
This means that the sharp complexity is closely related to the variational complexity of $g_1$. The more formal version of \eqref{eq:promise} of which a proof is in Appendix~\ref{app:proofs} reads
\begin{lemma}\label{lem:complexityforonechange} 
Let $t,\e \in (0,1), a,b\in \R$. Define $g_1(x) := b + a[x - t]_+$. If $k \leq M \leq \frac{1}{\sigma_w^2}, \sigma_b^2 \leq \frac{1}{\sigma_w^2}$ and $\Omega(\e^{1/4}),\Omega(\log(k/\sigma_w) \sigma_w^2) \leq |a| < 2, \Omega(\e^{1/4}) \leq  |b|, \Omega(\e^{1/2}) \leq \min(t,1-t)$ then
$$
\frac{|a|}{3 \sigma_w^2} \leq \chi^\#(g_1, U([0,1]), \e^2) \leq 2\left(\frac{|a|}{\sigma_w^2} + \frac{|b|}{\sigma_b^2} \right) + 11 - 3\log(\e).
$$
\end{lemma}

The above lemma is stated with the most general setting of parameters. To get more insight into the meaning of the lemma we give the following corollary.

\begin{example}\label{exm:onechange}
For every sufficiently small $\sigma_e^2$ and $M = k, \sigma_w^2 = \frac{1}{k}, \sigma_b^2 = 1, |b| = \Theta \left(\sigma_e^{1/2} \right)$, $\Omega \left( \sigma_e^{1/4} \right), \Omega \left(\frac{\log(k)}k \right) \leq |a| < 2$ if we define $g_1(x) := b + a[x-\frac12]_+$ then
$$
\chi^\#(g_1,U[0,1],\sigma_e^2) \leq 3|a|k + 3 \log \left(\frac{1}{\sigma_e} \right).
$$
\end{example}

\begin{proof}
One can easily verify that the assumptions of Lemma~\ref{lem:complexityforonechange} are satisfied. Applying the lemma we get
\begin{align*}
\chi^\#(g_1,U[0,1],\sigma_e^2) 
&\leq 2\left(\frac{|a|}{\sigma_w^2} + \frac{|b|}{\sigma_b^2} \right) + 11 + 3\log \left(\frac{1}{ 
\sigma_e}\right)  \\
&\leq 2|a| k + \Theta \left(\sigma_e^{1/2} \right) + 11 + 3\log \left(\frac{1}{ 
\sigma_e} \right) \\
&\leq 3|a|k + 3 \log \left(\frac{1}{\sigma_e} \right) && \text{As } \Omega \left( \frac{\log(k)}{k} \right) \leq |a|.
\end{align*}
\end{proof}

\paragraph{Generalization bound.} Now we want to understand what Example~\ref{exm:onechange} gives us for the generalization bound from Theorem~\ref{thm:generalizationbound}. Setting $\beta = 1$ in Theorem~\ref{thm:generalizationbound} and applying Example~\ref{exm:onechange},  we can bound
\begin{align}
&\expectation_{\samp \sim \dist^N}[L_\dist(Q)] \nonumber \\
&\leq
2\sigma_e^2 + \frac{C}{\sqrt{2}}\sqrt{\frac{\chi^\#(g_1, \dist_x, \sigma_e^2)}{N}} \nonumber \\
&\leq 2\sigma_e^2 + \frac{C}{\sqrt{2}}\sqrt{\frac{3|a|k + 3 \log \left(\frac{1}{\sigma_e} \right)}{N} } \label{eq:generalizationboundforonechange}.
\end{align}

Now we interpret \eqref{eq:generalizationboundforonechange}. First note that the setting of parameters in Example~\ref{exm:onechange} is natural. The choice of $\sigma_w^2 = \frac{1}{k}$ and $\sigma_b^2 = 1$ are among standard choices for initialization schemes. We pick $|b| = \Theta \left(\sigma_e^{1/2} \right)$ and $t = 1/2$ in order to analyze functions $g_1(x)\approx a\left[x - \frac12 \right]_+$, where the bias term $b$ is nonzero because of the assumptions of Lemma~\ref{lem:complexityforonechange}. Note that depending on the relation between $k$ and $\sigma_e^2$ one of the terms dominates \eqref{eq:generalizationboundforonechange}: either $3|a|k$ or $3 \log \left(\frac{1}{\sigma_e} \right)$. 

If $\sigma_e^2 \ll k$ then $3 \log \left(\frac{1}{\sigma_e} \right)$ dominates and the generalization bound depends mostly on the noise level $\sigma_e^2$ \footnote{As a side note, notice that the $3$ in $3 \log \left(\frac{1}{\sigma_e} \right)$ corresponds to the $2c+1$ bound on the limiting complexity in Example~\ref{lem:cchangeslimit}, as we consider a function with one change of slope and a very small $\e^2$ for computing $\chi^\#$. This illustrates the relationship between sharp and limiting complexity.}.

If $\sigma_e^2 \gg k$ then $3|a|k$ dominates. In this case we get the promised dependence of the generalization bound on $|a|$, which we recall is equal to $C(g_1)$. Note that there is a wide range of $|a|$ for which the bound holds, i.e. $ \Omega \left(\frac{\log(k)}{k} \right) \leq |a| \leq 2$. We see that the simpler $g_1$, measured in terms of $C$, the better a generalization bound we get. 

\subsection{Neural Networks with Several Hidden Layers} \label{sec:severalhiddenlayer}
Consider now exactly the same set-up as before, except that now we have $K=4$, i.e., we have three hidden layers and still $d = 1$. We can still represent piece-wise linear functions (e.g., by using the first layer to represent the function and just a single node in the second layer to sum the output of the previous layers). But the asymptotic complexity of some functions is now different! 

\begin{example}[Periodic function]
Imagine that we want to represent a function $g : [0,l] \rightarrow \R$ that is periodic with period $1$. That is $g(x - 1) = g(x)$ for all $x \in [1,l]$. What we can do is to (i) represent $g|_{[0,1]}$ in the output of a single neuron $v$ in layer $2$ (ii) represent shifted versions of $g|_{[0,1]}$ (which are equal to $g|_{[1,2]}, g|_{[2,3]}, \dots$ due to periodicity) in the outputs of neurons in layer $3$ (iii) sum outputs of neurons from layer $3$ in the single output neuron in layer $4$. Assume moreover that $g|_{[0,1]}$ has $m$ changes of slope. Then observe that we implemented $g$ fixing $O(l+m)$ degrees of freedom. But $g$ itself has $m \cdot l$ changes of slope over the whole domain. 

This representation gives an upper-bound for limiting complexity as there might be other ways to represent the function. 
But because of Example~\ref{exm:onechange} it is enough to arrive at a separation. Indeed if $l \approx m$ then the asymptotic complexity of $g$ for NN with $4$ layers is smaller than for $2$ layers, which is in $\Omega(m l)$. In words, we obtain a quadratic gain in terms of the number of samples needed to get the same generalization bound.

We leave it for future work to explore this direction in more depth (no pun intended).
\end{example}

\section{Why Linear Schemes Generalize Poorly}\label{sec:linear}

In Section~\ref{sec:nngeneralizewell} we've seen that for NNs our notion of complexity aligns well with natural notions of complexity. This, in the light of the connections to the PAC-Bayes bound, partly explains their good generalization. In this section we will show that for the case of linear schemes the complexity is basically independent of a function.  

We investigate  \label{sec:linearmodel}
the linear model $\model_\theta^{(\text{L, o})}=\{f_{\theta}(x): f_{\theta}(x) = \sum_{i=0}^{d-1} \w_i b_i(x), x \in \inspace = [-1, 1]\}$. Further let $\dist_x$ be the uniform distribution on $[-1, 1]$. We assume a prior on $\w_i$'s to be iid Gaussians of mean $0$ and variance $\sigma_w^2$.

We will see that in this setting all realizable functions have the same complexity. This in the light of \eqref{eq:generalizationbound} tells us that even if reality prefers simple functions the number of samples needed to get a non vacuous bound is as big as the one needed for the highest complexity function in the class.  In short: it is equally ``easy'' to represent a ``complicated'' function as it is to represent a ``simple'' one. Therefore, given some samples, there is no reason to expect that linear models will fit a simple function to the data. Indeed, to the contrary. If the data is noisy, then linear models will tend to overfit this noise.

\subsection{Orthonormal Basis} \label{sec:Lo}

For simplicity assume that the basis functions are the Legendre polynomials. I.e., we start with the polynomials $\{1, x, x^2, ...\}$ and then create from this an orthonormal set on $[-1, 1]$ via the Gram-Schmidt process.

\begin{example}[Constant Function]
Let $g(x)=\frac{1}{\sqrt{2}}$. This function is realizable. Indeed,
it is equal to the basis function $b_0(x)$.  Let us compute $\chi^\#(\model^{(\text{L, o})},g,
\epsilon^2)$.  If we pick all weights in $f_{\w}(x) = \sum_{i=0}^{d-1}
\w_i b_i(x)$ equal to $0$ except $\w_0$ equal to $1$ then we get
$g(x)$.  Hence, taking advantage of the fact that the basis functions
are orthonormal, we have
\begin{align*}
&\expectation_{x \sim \dist_x}[(f_{\w} - g(x))^2]  =
\frac12 \int_{-1}^{1} (f_{\w}(x)-g(x))^2 dx  \\  
= & \frac12 \sum_{i=0}^{d-1} (\w_i-\ind_{\{i=0\}})^2 \int_{-1}^{1} b_i(x)^2 dx  
= \frac12 \sum_{i=0}^{d-1} (\w_i-\ind_{\{i=0\}})^2.
\end{align*}
So we need to compute the probability 
\begin{align*}
\prob\left[\w:  \frac12 \sum_{i=0}^{d-1} (\w_i-\ind_{\{i=0\}})^2 \leq \e^2\right].
\end{align*}
Recall that our weights are iid Gaussians of mean $0$ and variance $\sigma_w^2$. Hence
\begin{align*}
\sum_{i=1}^{d-1} \w_i^2  \sim \Gamma \left(\frac{d-1}{2}, 2 \sigma_w^2 \right),
\end{align*}
where $\Gamma(k, \theta)$ denotes the Gamma distribution with shape
$k$ and scale $\theta$.  It follows that the probability we are
interested in can be expressed as $\prob\left[\w:  \frac12 \sum_{i=0}^{d-1} (\w_i-\ind_{\{i=0\}})^2 \leq \e^2\right] = q(\kappa=1, \sigma_w, \epsilon)$, where
$$q(\kappa, \sigma_w, \epsilon) =
\frac{1}{\sqrt{2 \pi \sigma_w^2}}\int_{0}^{\epsilon} F(\epsilon^2-x^2; \frac{d-1}{2}, 2 \sigma_w^2) \left[e^{-\frac{(\kappa+x)^2}{2 \sigma_w^2}}+e^{-\frac{(\kappa-x)^2}{2 \sigma_w^2}} \right]dx.
$$
Here, $F(x; k, \theta)$ denotes the cdf of the Gamma distribution
with shape $k$ and scale $\theta$. From the above expression we can compute $\chi^\#(\model^{(\text{L, o})}, g(x) =
\frac{1}{\sqrt{2}},\epsilon^2)$, although there does not seem to be an elementary expression.

\begin{lemma}[$q(\kappa, \sigma_w, \epsilon)$] \label{lem:qproperties}
For non-negative $\kappa$, $\sigma_w$, and $\epsilon \in (0, 1]$ the function
$q(\alpha, \sigma_w, \epsilon)$ has the following properties:
\begin{enumerate}[(i)]
\item Scaling: $q(\kappa, \sigma_w, \epsilon) = \kappa q(1, \sigma_w/\kappa, \epsilon/\kappa)$
\item Monotonicity in $\kappa$: $q(\kappa, \sigma_w, \epsilon)$ is non-increasing in $\kappa$
\item Monotonicity in $\sigma_w$: $q(\kappa, \sigma_w, \epsilon)$ is non-increasing in $\sigma_w$
\item Monotonicity in $\epsilon$: $q(\kappa, \sigma_w, \epsilon)$ is non-decreasing in $\epsilon$
\item Limit: $\lim_{\epsilon \rightarrow 0}\log(q(\kappa, \sigma_w, \epsilon))/\log(\epsilon)=d$
\end{enumerate}
\end{lemma}

If we are just interested in $\chi^\#(\model^{(\text{L, o})}, g(x) = \frac{1}{\sqrt{2}})$,
we can start from $\chi^\#(\model^{(\text{L, o})}, g(x) = \frac{1}{\sqrt{2}}, \epsilon^2)$
or we can make use of (v) of Lemma~\ref{lem:qproperties} to get 
\begin{equation}
\chi^\#(\model^{(\text{L, o})}, g(x) = \frac{1}{\sqrt{2}})=d.    
\end{equation}
To see this result intuitively note that all weights
have to be fixed to a definite value in order to realize $g(x)$.
\end{example}

\begin{example}[Basis Function]
Although we assumed in the above derivation that $g(x)=b_0(x)$ the
calculation is identical for any $g(x)=b_i(x)$, $i=0, \cdots, d-1$.
We conclude that $\chi^\#(\model^{(\text{L, o})}, b_i(x))$ does not depend
on $i$.  \end{example}

\begin{example}[Realizable Function of Norm $1$]
Assume that $g(x)= \sum_{i=0}^{d-1} \widetilde{\w}_i b_i(x)$ with
$\sum_{i=0}^{d-1} \widetilde{\w}_i^2=1$. In other words, the function
is realizable and has squared norm equal to $1$.

If we ``rotate'' (orthonormal transform) our basis $\{b_i(x)\}_{i=0}^d$
into the new basis $\{\tilde{b}_i(x)\}_{i=0}^d$ so that
$g(x)=\tilde{b}_0(x)$ then due to the rotational symmetry of our
prior we are back to our first example.

We conclude that for any realizable function $g(x)$ of norm $1$,
$$\chi^\#(\model^{(\text{L, o})}, g(x), \epsilon^2) = \chi^\#(\model^{(\text{L, o})}, b_0(x), \epsilon^2).$$
\end{example}

\begin{example}[Realizable Function]
Assume that $g(x)= \sum_{i=0}^{d-1} \widetilde{\w}_i b_i(x)$ with
$\sum_{i=0}^{d-1} \widetilde{\w}_i^2=\kappa^2$. In other words, the
function is realizable and has norm equal to $\kappa$.

Using the scaling property of Lemma~\ref{lem:qproperties} we can write 
\begin{align*}
&\chi^\#\left(\model_{\w(\sigma_w)}^{(\text{L, o})}, g(x), \epsilon^2 \right) \\
& = -\log(q(\kappa, \sigma_w, \epsilon)) \\
& = -\log(\kappa q(1, \sigma_w/\kappa, \epsilon/\kappa)) \\
& = -\log(\kappa) + \chi^\#(\model_{\w(\sigma_w/\kappa)}^{(\text{L, o})}, b_0(x), \epsilon^2/\kappa^2),
\end{align*}
where we wrote $\model_{\w(\sigma_w)}^{(\text{L, o})}$ to indicate that in the model
each parameter's prior is a Gaussian with variance $\sigma_w^2$.

This means that the complexity of a function changes depending on the norm of the vector of weights $\w$ that represent it. However if we are interested in the asymptotic complexity all functions (apart from the $0$ function) have the same complexities as $\lim_{\e \rightarrow 0} \frac{\log(\kappa) }{\log(\e)} = 0$, which leads to the next example.
\end{example}

\begin{example}[Limiting Sharp Complexity]\label{ex:linearlimit}
Assume that $g(x)= \sum_{i=0}^{d-1} \widetilde{\w}_i b_i(x)$. Then
$$
\chi^\#(\model^{(\text{L, o})}, g(x))=d.
$$
\end{example}

Recall that we showed (Example~\ref{lem:cchangeslimit}) that for the case of 2-layer neural networks the limiting complexity depends strongly on the function and simpler functions - in a sense of number of changes of slope - have lower complexity. Here we see that for linear models basically all functions have the same complexity, which is equal to the number of basis functions in the model.

\begin{example}[Unrealizable Function]
Given any function $g(x)$, we can represent it as
$g(x)=g_{\perp}(x)+g_{\|}(x)$, where the two components are orthogonal
and where $g_{\|}(x)$ represents the realizable part. We then have
that $\chi^\#(\model_{\w(\sigma_w)}^{(\text{L, o})}, g(x), \epsilon^2)$ is equal to
\begin{align*}
\begin{cases}
\infty, & \|g_{\perp}(x)\|_2^2 > \epsilon^2, \\
-\log \left(q \left(1, \sigma_w, \sqrt{\epsilon^2-\|g_{\perp}(x)\|_2^2} \right) \right), & 
\|g_{\perp}(x)\|_2^2 < \epsilon^2.
\end{cases}
\end{align*}
\end{example}

\subsection{Non-Orthonormal Basis}
\begin{example}[Non-Orthogonal Basis]
If the functions do not form an orthonormal basis but are independent, then we
can transform them into such base. After the transform the probability distribution is
still a Gaussian but no longer with independent components. Now the
"equal simplicity" lines are ellipsoids.

And if we have dependent components then we also still have Gaussians
but we are in a lower dimensional space.
\end{example}

\subsection{Summary}
We have seen that for the linear model the complexity of a function
$g(x)$ only depends on the norm of the signal. This complexity measure is therefore only
weakly correlated with other natural complexity measures. E.g., if
the basis consists of polynomials of increasing degrees and the reality is modeled by a function of low degree then the bound from \eqref{eq:generalizationbound} is the same as when the reality is modeled by a high degree polynomial. It means that the number of samples needed for a good generalization bound is independent of $g$.



\bibliography{references.bib}

\begin{thebibliography}{17}
\providecommand{\natexlab}[1]{#1}
\providecommand{\url}[1]{\texttt{#1}}
\expandafter\ifx\csname urlstyle\endcsname\relax
  \providecommand{\doi}[1]{doi: #1}\else
  \providecommand{\doi}{doi: \begingroup \urlstyle{rm}\Url}\fi

\bibitem[Alquier(2021)]{pacbayes}
Pierre Alquier.
\newblock User-friendly introduction to pac-bayes bounds, 10 2021.

\bibitem[Arora et~al.(2018)Arora, Ge, Neyshabur, and Zhang]{pmlr-v80-arora18b}
Sanjeev Arora, Rong Ge, Behnam Neyshabur, and Yi~Zhang.
\newblock Stronger generalization bounds for deep nets via a compression
  approach.
\newblock In Jennifer Dy and Andreas Krause, editors, \emph{Proceedings of the
  35th International Conference on Machine Learning}, volume~80 of
  \emph{Proceedings of Machine Learning Research}, pages 254--263. PMLR, 10--15
  Jul 2018.
\newblock URL \url{http://proceedings.mlr.press/v80/arora18b.html}.

\bibitem[Baldassi et~al.(2015)Baldassi, Ingrosso, Lucibello, Saglietti, and
  Zecchina]{zecchina1}
Carlo Baldassi, Alessandro Ingrosso, Carlo Lucibello, Luca Saglietti, and
  Riccardo Zecchina.
\newblock Subdominant dense clusters allow for simple learning and high
  computational performance in neural networks with discrete synapses.
\newblock \emph{Physical review letters}, 115 12:\penalty0 128101, 2015.

\bibitem[Baldassi et~al.(2016)Baldassi, Borgs, Chayes, Ingrosso, Lucibello,
  Saglietti, and Zecchina]{zecchina2}
Carlo Baldassi, Christian Borgs, Jennifer~T. Chayes, Alessandro Ingrosso, Carlo
  Lucibello, Luca Saglietti, and Riccardo Zecchina.
\newblock Unreasonable effectiveness of learning neural networks: From
  accessible states and robust ensembles to basic algorithmic schemes.
\newblock \emph{Proceedings of the National Academy of Sciences}, 113:\penalty0
  E7655 -- E7662, 2016.

\bibitem[Bartlett et~al.(2017)Bartlett, Foster, and
  Telgarsky]{NIPS2017_b22b257a}
Peter~L Bartlett, Dylan~J Foster, and Matus~J Telgarsky.
\newblock Spectrally-normalized margin bounds for neural networks.
\newblock In I.~Guyon, U.~V. Luxburg, S.~Bengio, H.~Wallach, R.~Fergus,
  S.~Vishwanathan, and R.~Garnett, editors, \emph{Advances in Neural
  Information Processing Systems}, volume~30. Curran Associates, Inc., 2017.
\newblock URL
  \url{https://proceedings.neurips.cc/paper/2017/file/b22b257ad0519d4500539da3c8bcf4dd-Paper.pdf}.

\bibitem[Chen et~al.(2016)Chen, Ding, and Carin]{chen2016convergence}
Changyou Chen, Nan Ding, and Lawrence Carin.
\newblock On the convergence of stochastic gradient mcmc algorithms with
  high-order integrators, 2016.

\bibitem[Dziugaite and Roy(2017)]{dziugaite}
Gintare~Karolina Dziugaite and Daniel~M. Roy.
\newblock Computing nonvacuous generalization bounds for deep (stochastic)
  neural networks with many more parameters than training data.
\newblock In Gal Elidan, Kristian Kersting, and Alexander~T. Ihler, editors,
  \emph{Proceedings of the Thirty-Third Conference on Uncertainty in Artificial
  Intelligence, {UAI} 2017, Sydney, Australia, August 11-15, 2017}. {AUAI}
  Press, 2017.
\newblock URL \url{http://auai.org/uai2017/proceedings/papers/173.pdf}.

\bibitem[Hinton and van Camp(1993)]{hintonflatminima}
Geoffrey~E. Hinton and Drew van Camp.
\newblock Keeping the neural networks simple by minimizing the description
  length of the weights.
\newblock In \emph{Proceedings of the Sixth Annual Conference on Computational
  Learning Theory}, COLT '93, page 5–13, New York, NY, USA, 1993. Association
  for Computing Machinery.
\newblock ISBN 0897916115.
\newblock \doi{10.1145/168304.168306}.
\newblock URL \url{https://doi.org/10.1145/168304.168306}.

\bibitem[Hochreiter and Schmidhuber(1997)]{schmidhuberflatminima}
Sepp Hochreiter and J\"{u}rgen Schmidhuber.
\newblock Flat minima.
\newblock \emph{Neural Comput.}, 9\penalty0 (1):\penalty0 1–42, January 1997.
\newblock ISSN 0899-7667.
\newblock \doi{10.1162/neco.1997.9.1.1}.
\newblock URL \url{https://doi.org/10.1162/neco.1997.9.1.1}.

\bibitem[Marceau-Caron and Ollivier(2017)]{marceaucaron2017natural}
Gaétan Marceau-Caron and Yann Ollivier.
\newblock Natural langevin dynamics for neural networks, 2017.

\bibitem[Nagarajan and Kolter(2019)]{NEURIPS2019_05e97c20}
Vaishnavh Nagarajan and J.~Zico Kolter.
\newblock Uniform convergence may be unable to explain generalization in deep
  learning.
\newblock In H.~Wallach, H.~Larochelle, A.~Beygelzimer, F.~d\textquotesingle
  Alch\'{e}-Buc, E.~Fox, and R.~Garnett, editors, \emph{Advances in Neural
  Information Processing Systems}, volume~32. Curran Associates, Inc., 2019.
\newblock URL
  \url{https://proceedings.neurips.cc/paper/2019/file/05e97c207235d63ceb1db43c60db7bbb-Paper.pdf}.

\bibitem[Neyshabur et~al.(2015)Neyshabur, Tomioka, and
  Srebro]{DBLP:journals/corr/NeyshaburTS15}
Behnam Neyshabur, Ryota Tomioka, and Nathan Srebro.
\newblock Norm-based capacity control in neural networks.
\newblock \emph{CoRR}, abs/1503.00036, 2015.
\newblock URL \url{http://arxiv.org/abs/1503.00036}.

\bibitem[Neyshabur et~al.(2017{\natexlab{a}})Neyshabur, Bhojanapalli,
  McAllester, and Srebro]{DBLP:journals/corr/NeyshaburBMS17}
Behnam Neyshabur, Srinadh Bhojanapalli, David McAllester, and Nathan Srebro.
\newblock Exploring generalization in deep learning.
\newblock \emph{CoRR}, abs/1706.08947, 2017{\natexlab{a}}.
\newblock URL \url{http://arxiv.org/abs/1706.08947}.

\bibitem[Neyshabur et~al.(2017{\natexlab{b}})Neyshabur, Bhojanapalli,
  McAllester, and Srebro]{DBLP:journals/corr/NeyshaburBMS17aa}
Behnam Neyshabur, Srinadh Bhojanapalli, David McAllester, and Nathan Srebro.
\newblock A pac-bayesian approach to spectrally-normalized margin bounds for
  neural networks.
\newblock \emph{CoRR}, abs/1707.09564, 2017{\natexlab{b}}.
\newblock URL \url{http://arxiv.org/abs/1707.09564}.

\bibitem[Savarese et~al.(2019)Savarese, Evron, Soudry, and
  Srebro]{srebronormbound}
Pedro Savarese, Itay Evron, Daniel Soudry, and Nathan Srebro.
\newblock How do infinite width bounded norm networks look in function space?
\newblock In Alina Beygelzimer and Daniel Hsu, editors, \emph{Proceedings of
  the Thirty-Second Conference on Learning Theory}, volume~99 of
  \emph{Proceedings of Machine Learning Research}, pages 2667--2690, Phoenix,
  USA, 25--28 Jun 2019. PMLR.
\newblock URL \url{http://proceedings.mlr.press/v99/savarese19a.html}.

\bibitem[Shekhtman(1982)]{shekhtman82}
Boris Shekhtman.
\newblock Why piecewise linear functions are dense in c[0,1].
\newblock \emph{Journal of approximation theory}, 1982.

\bibitem[Teh et~al.(2015)Teh, Thiéry, and Vollmer]{teh2015consistency}
Yee~Whye Teh, Alexandre Thiéry, and Sebastian Vollmer.
\newblock Consistency and fluctuations for stochastic gradient langevin
  dynamics, 2015.

\end{thebibliography}

\clearpage
\clearpage

\appendix



\section{Generalization bound}\label{apx:generalization}

To derive the bound from Theorem~\ref{thm:generalizationbound} in terms of ``sharp complexity'' we first define a series of related notions that are helpful during the derivation.

We define the {\em empirical} complexity of a function $g$ as 
\begin{align*}
&\chi^E(g, \samp_x, \samp_{\epsilon}, \sigma_y^2) \\ 
&:= -\log \left[ \left( \int_{\theta} P(\theta) e^{-\frac{1}{2\sigma_y^2N} \sum_{n=1}^{N} (g(x_n) + \eta_n - f_{\theta}(x_n))^2} d \theta\right) \right],
\end{align*}
where we denoted by $\samp_x$ the $x$'s part of $\samp$ and by $\samp_\e$ the particular realization of noise used for generating $\samp$, i.e. $\eta$'s.

In order to compute it, we integrate over the parameter space and weigh the prior $P(\theta)$ by an exponential factor which is the smaller the further the function $f_{\theta}$ is from $g$ on the given sample $\samp_x$ plus noise $\samp_\e$. Recall that noise samples $\samp_{\epsilon}$ come from an iid Gaussian zero-mean sequence of variance $\sigma_e^2$. We then take the negative logarithm of this integral.

The {\em true complexity with noise} is defined as
\begin{align*}
&\chi^N(g, \dist_x, \sigma_y^2, \sigma_\e^2) := \\
&-\log \left[ \left( \int_{\theta} P(\theta) e^{-\frac{1}{2\sigma_y^2} \expectation_{x \sim \dist_x, \noise \sim \mathcal{N}(0,\sigma_e^2)} [(g(x) + \noise - f_{\theta}(x))^2]} d \theta\right) \right],   
\end{align*}
where the sum has been replaced by an expectation using the underlying distribution of the input.

The {\em exponential complexity} is
\begin{align*}
&\chi(g, \dist_x, \sigma_y^2) := \\
&-\log \left[ \left( \int_{\theta} P(\theta) e^{-\frac{1}{2\sigma_y^2} \expectation_{x \sim \dist_x} [(g(x) - f_{\theta}(x))^2]} d \theta\right) \right].   
\end{align*}
Note that
\begin{align*}
\chi(g, \dist_x, \sigma_y^2) + \frac{ \sigma_e^2}{2 \sigma_y^2} = \chi^N(g, \dist_x, \sigma_y^2, \sigma_e^2).
\end{align*}

Finally, the {\em sharp complexity with noise} is defined as
\begin{align*}
&\chi^{\#N}(g, \dist_x, \sigma_e^2, \e^2) \\
&:= -\log \left[ \prob_\theta[ \expectation_{x \sim \dist_x, \noise \sim \mathcal{N}(0,\sigma_e^2)} [(g(x) + \noise - f_{\theta}(x))^2] \leq \e^2] \right].
\end{align*}

The following two lemmas establish some relationships between these notions of complexity.
\begin{lemma}[Sharp complexity: with and without noise]\label{lem:sharpwithnoisenonoise}
For every $\dist_x$, every $g : \inspace \rightarrow \outspace$, and $\e^2 > 0$ we have:
\begin{align*}
&\chi^{\#N}(g, \dist_x, \sigma_e^2, \e^2) = \chi^\#(g, \dist_x, \e^2 - \sigma_e^2).
\end{align*}
\end{lemma}

\begin{proof}
\begin{align*}
&\chi^{\#N}(g, \dist_x, \sigma_e^2, \e^2) \\
& = -\log \left[ \prob_\theta \left[ \expectation_{x \sim \dist_x, \noise \sim \mathcal{N}(0,\sigma_e^2)} [(g(x) + \noise - f_{\theta}(x))^2] \leq \e^2 \right] \right] \\
&= -\log \left[ \prob_\theta \left[ \expectation_{x \sim \dist_x, \noise \sim \mathcal{N}(0,\sigma_e^2)} [(g(x) - f_{\theta}(x))^2] \leq \e^2 - \sigma_e^2 \right] \right] \\
&= \chi^\#(g, \dist_x, \e^2 - \sigma_e^2),
\end{align*}
where in the second equality we write $(g(x) +\noise - f_\theta(x))^2$ as the sum of $(g(x)-f_\theta(x))^2$, $2\noise(g(x) - f_\theta(x))$ and $\noise^2$ and use the fact that $\expectation[\noise] = 0$ and $\expectation[\noise^2] = \sigma_e^2$.
\end{proof}

\begin{lemma}[True versus sharp complexity] \label{lem:trueversussharp}
For every $\dist_x$, every $g : \inspace \rightarrow \outspace$, and $\sigma_y^2, \sigma_e^2, \e^2 > 0$ we have:
\[
\chi^N(g, \dist_x, \sigma_y^2, \sigma_e^2) \leq \chi^{\#N}(g, \dist_x, \sigma_e^2, \e^2) + \frac{\e^2}{2\sigma_y^2}.
\]
\end{lemma}

\begin{proof}
\begin{align*}
&\chi^{\#N}(g, \sigma_e^2,\e^2) \\
&=
-\log \left( \int_{\theta} P(\theta) \mathds{1} \left\{ \expectation_{x \sim \dist_x, \noise \sim \mathcal{N}(0,\sigma_e^2)} [(g(x) + \noise - f_{\theta}(x))^2] \leq \e^2 \right\} d \theta \right) \\
&\stackrel{\alpha>0}{=} -\log \left( \int_{\theta} P(\theta) \mathds{1} \left\{ \frac{\alpha}{2\sigma_y^2} \expectation_{x \sim \dist_x, \noise \sim \mathcal{N}(0,\sigma_e^2)} [(g(x) + \noise - f_{\theta}(x))^2] \leq \frac{\alpha \e^2}{2\sigma_y^2} \right\} d \theta \right) \\
&\stackrel{e^{ x} \geq \mathds{1}\{x \geq 0 \}}{\geq} 
\!\!\!\!-\log \left( \int_{\theta} P(\theta) e^{\frac{\alpha \e^2}{2\sigma_y^2} -\frac{\alpha}{2\sigma_y^2} \expectation_{x \sim \dist_x, \noise \sim \mathcal{N}(0,\sigma_e^2)} [(g(x) + \noise - f_{\theta}(x))^2]} d \theta \right) \\
&= \chi^N(g, \sigma_y^2/\alpha, \sigma_e^2) - \frac{\alpha \e^2}{2\sigma_y^2}.
\end{align*}
\end{proof}

The sharp complexity is very convenient to work with. Hence we will formulate our final bound in terms of the sharp complexity. The reason we call it {\em sharp} complexity is that the region of $\theta$ we integrate over is defined by an indicator function whereas for the true complexity the ``boundary'' of integration is defined by a smooth function.

Let us now look more closely at the divergence where we assume the data model (\ref{equ:datamodel}) and that the true hypothesis is $g$. We have
\begin{align}
&D(Q \| P) \nonumber \\
&= \int \frac{P(\theta) e^{- \frac{1}{2 \sigma_y^2} \sum_{n=1}^{N} (y_n - f_{\theta}(x_n))^2}}{\int P(\theta') e^{- \frac{1}{2 \sigma_y^2} \sum_{n=1}^{N} (y_n - f_{\theta'}(x_n))^2} d \theta'} \cdot \nonumber \\
&\cdot \log \left(
\frac{ e^{- \frac{1}{2 \sigma_y^2} \sum_{n=1}^{N} (y_n - f_{\theta}(x_n))^2}}{\int P(\theta') e^{- \frac{1}{2 \sigma_y^2} \sum_{n=1}^{N} (y_n - f_{\theta'}(x_n))^2} d \theta'} 
\right) d \theta \nonumber \\
 &\leq \chi^E(g,\samp_x, \samp_{\epsilon},\sigma_y^2/N)- \frac{N}{2\sigma_y^2} L_\samp(Q), 
 \label{eq:divergenceupperbound}
 \end{align}
where in the last inequality we used the fact that we use a clipped version of a square loss.
Therefore the expectation over $S \sim \dist^N$ of the square root term of the right-hand side of the PAC-Bayes bound \eqref{equ:pacbound} can be upper-bounded as
\begin{align}
&\expectation_{\samp \sim \dist^{N}} \left[C\sqrt{\frac{D(Q \| P)}{2 N}} \right] \nonumber \\
&\stackrel{\text{By } \eqref{eq:divergenceupperbound}}{\leq} \expectation_{\samp \sim \dist^{N}} \left[C \sqrt{\frac{\chi^E(g, \samp_x,\samp_\e, \sigma^2_y/N) - \frac{N}{2\sigma_y^2} L_\samp(Q) }{2 N}} \right] \nonumber \\
&\stackrel{\sqrt{\cdot} \text{ concave}}{\leq} 
\frac{C}{\sqrt{2}}\sqrt{\frac{\expectation_{\samp \sim \dist^{N}} \left[\chi^E(g, \samp_x,\samp_\e, \sigma^2_y/N) \right]}{N} -\frac{\widehat{L}}{2\sigma_y^2}} \label{eq:generalboundempcomp},
\end{align}
where we denoted $\expectation_{\samp \sim \dist^N}[L_\samp(Q)]$ by $\hat{L}$. Before we proceed we state a helpful lemma.

\begin{lemma} \label{lem:megaineq}
Let $X$ and $Y$ be independent random variables and $f(X, Y)$ be a non-negative function. Then
\begin{align*}
\expectation_X \left[ \ln \left( \expectation_Y \left[ e^{-f(X, Y)}\right] \right)  \right]
\geq \ln \left( \expectation_Y\left[e^{-\expectation_X[f(X, Y)]}\right]\right).
\end{align*}
\end{lemma}
\begin{proof}
We limit our proof to the simple case where the distributions are discrete and have a finite support, lets say from $\{1, \cdots, I\}$. We claim that for $1 \leq j <I$,
\begin{align*}
& (\sum_{i=1}^{j} p(X=x_i))
\ln \left( \expectation_Y\left[e^{-\frac{\sum_{i=1}^{j}p(X=x_i) f(x_i, Y)}{\sum_{i=1}^{j} p(X=x_i) }}\right]\right) + 
\sum_{i=j+1}^{I} p(X=x_i) \left[ \ln \left( \expectation_Y \left[ e^{-f(x_{i}, Y)}\right] \right)  \right] \\
\geq &
 (\sum_{i=1}^{j+1} p(X=x_i))
\ln \left( \expectation_Y\left[e^{-\frac{\sum_{i=1}^{j+1} p(X=x_i) f(x_i, Y)}{\sum_{i=1}^{j+1} p(X=x_i) }}\right]\right) +
\sum_{i=j+2}^{I} p(X=x_i) \left[ \ln \left( \expectation_Y \left[ e^{-f(x_{i}, Y)}\right] \right)  \right].
\end{align*}
This gives us a chain of inequalities. Note that the very first term in this chain is equal to the left-hand side of the desired inequality and the very last term is equal to the right-hand side of the inequality. 

Consider the $j$-th such inequality. Cancelling common terms, it requires us to prove
\begin{align*}
& (\sum_{i=1}^{j} p(X=x_i))
\ln \left( \expectation_Y\left[e^{-(\sum_{i=1}^{j} \frac{p(X=x_i) f(x_i, Y)}{\sum_{i=1}^{j} p(X=x_i) })}\right]\right) + 
 p(X=x_{j+1}) \left[ \ln \left( \expectation_Y \left[ e^{-f(x_{j+1}, Y)}\right] \right)  \right] \\
\geq &
 (\sum_{i=1}^{j+1} p(X=x_i))
\ln \left( \expectation_Y\left[e^{-(\sum_{i=1}^{j+1} \frac{p(X=x_i) f(x_i, Y)}{\sum_{i=1}^{j+1} p(X=x_i) })}\right]\right).
\end{align*}
Taking the prefactors into the logs, combining the two log terms on the left-hand side, and finally cancelling the logs, the claimed inequality is true iff
\begin{align*}
\expectation_Y\!\!\! \left[e^{-\frac{\sum_{i=1}^{j} p(X=x_i) f(x_i, Y)}{\sum_{i=1}^{j} p(X=x_i) }}\right]^{\frac{\sum_{i=1}^{j} p(X=x_i)}{\sum_{i=1}^{j+1} p(X=x_i)}} 
   \!\!\!\!\!\! \expectation_Y \!\!\! \left[ e^{-f(x_{j+1}, Y)}\right] ^{\frac{p(X=x_{j+1})}{\sum_{i=1}^{j+1} p(X=x_i)}}
\geq &
\expectation_Y \!\!\! \left[e^{-\frac{\sum_{i=1}^{j+1} p(X=x_i) f(x_i, Y)}{\sum_{i=1}^{j+1} p(X=x_i) }}\right].
\end{align*}
But this statement is just an instance of the Hoelder inequality with $1/p+1/q=1$, where $1/p=\frac{\sum_{i=1}^{j} p(X=x_i)}{\sum_{i=1}^{j+1} p(X=x_i)}$ and $1/q=\frac{p(X=x_{j+1})}{\sum_{i=1}^{j+1} p(X=x_i)}$.
\end{proof}

Now we bound the complexity term from \eqref{eq:generalboundempcomp} further. We have for every $\e^2 > 0$
\begin{align}
&\expectation_{S \sim \dist^{N}}[\chi^E(g, \samp_x,\samp_\e, \sigma_y^2/ N)] \nonumber \\
&= -\expectation_{S \sim \dist^{N}} \left[ \log  \left( \int_{\theta} P(\theta) e^{-\frac{1}{2\sigma^2_y} \sum_{n=1}^{N} (g(x_n) + \e_n  - f_{\theta}(x_n))^2} d \theta\right)   \right] \nonumber \\
&\stackrel{\text{Lem~\ref{lem:megaineq}}}{\leq} -\log \left( \int_{\theta} P(\theta) e^{-\frac{N}{2\sigma^2_y} \expectation_{\stackrel{x \sim \dist_x}{ \noise \sim \mathcal{N}(0,\sigma_e^2)}} [(g(x) + \noise - f_\theta(x))^2]} d \theta\right) \nonumber \\
&= \chi^N(g, \dist_x, \sigma^2_y/N, \sigma_e^2)
\stackrel{\text{Lem~\ref{lem:trueversussharp}}}{\leq} \chi^{\#N}(g, \dist_x, \sigma_e^2, \e^2) + \frac{\e^2 N }{2 \sigma^2_y} \nonumber \\
&\stackrel{\text{Lem}~\ref{lem:sharpwithnoisenonoise}}{=} \chi^\#(g, \dist_x, \e^2 - \sigma_e^2) + \frac{\e^2 N }{2 \sigma^2_y}. \label{eq:generalboundtruecomp} 
\end{align}
Hence by combining \eqref{eq:generalboundempcomp} and \eqref{eq:generalboundtruecomp} we get that for every $\e^2 > 0$ the expectation over $S \sim \dist^N$ of the PAC-Bayes bound can be bounded as
\begin{align}
&\expectation_{\samp \sim \dist^{N}}\left[L_{\samp}(Q) + C\sqrt{\frac{D(Q \| P) }{2 N}}\right] \nonumber \\
&\leq \widehat{L} + \frac{C}{\sqrt{2}}\sqrt{\frac{\chi^\#(g, \dist_x, \e^2 - \sigma_e^2)}{N}  + \frac{1}{2\sigma^2_y}\left(\e^2 - \widehat{L}\right)} \label{eq:generalizationbound}.
\end{align}
Let $\beta \in (0,1]$. Recall that parameter $\sigma_y^2$ is chosen freely by the learning algorithm. By the assumption of the theorem we have 
\begin{equation}\label{eq:lwrbndonLd}
L_\dist(P) \geq 2\sigma_e^2. 
\end{equation}
Because $g \in \text{supp}(P)$, which in words means that $g$ is realizable with prior $P$, then
\begin{align}
\lim_{\sigma_y^2 \rightarrow 0} \widehat{L}
&= \lim_{\sigma_y^2 \rightarrow 0} \expectation_{\samp \sim \dist^N} [ L_\samp(Q) ]  \nonumber\\
&= \expectation_{\samp \sim \dist^N} \left[ \lim_{\sigma_y^2 \rightarrow 0} L_\samp(Q) \right]  \nonumber\\
&\leq \expectation_{\samp \sim \dist^N} L_\samp(g) \nonumber\\
&= \sigma_e^2 . \label{eq:Lhatzero}
\end{align}

where in the second equality we used Lebesgue dominated convergence theorem and in the inequality we used the fact that the smaller $\sigma_y^2$ gets the bigger the penalty on $\sum_n (y_n - f_\theta(x_n))^2$ in $Q$, which means that, in the limit, $L_\samp(Q)$ is smaller than $L_\samp(h)$ for every fixed $h \in \text{supp}(P)$ and in particular for $g$. 

On the other hand, by an analogous argument, we have
\begin{align}
\lim_{\sigma_y^2 \rightarrow \infty} \widehat{L} 
&= \expectation_{\samp \sim \dist^N} \left[ L_\samp(P) \right] \nonumber \\
&= \expectation_{\samp \sim \dist^N} \left[  \expectation_{\theta \sim P} \left[\frac{1}{N} \sum_{i=1}^N \ell(f_\theta, y_n) \right] \right] \nonumber \\
&=  \expectation_{\theta \sim P} \left[ \expectation_{\samp \sim \dist^N} \left[\frac{1}{N} \sum_{i=1}^N \ell(f_\theta, y_n)  \right] \right] \nonumber \\
&= L_\dist(P)  \nonumber \\
&\geq 2 \sigma_e^2, \label{eq:Lhatinfty}
\end{align}
where we used the independence of $P$ and $\samp$ in the third equality and \eqref{eq:lwrbndonLd} in the inequality.

Equations \eqref{eq:Lhatzero} and \eqref{eq:Lhatinfty} and the fact that $\widehat{L}$ is a continuous function of $\sigma_y^2$ give us that there exists $\sigma_{\text{alg}}^2 > 0$ such that
$$ 
\expectation_{\samp \sim \dist^N} \left[L_\samp(Q(\sigma_{\text{alg}}^2)) \right] = (1 + \beta) \sigma_e^2,
$$
where we wrote $Q(\sigma_{\text{alg}}^2)$ to explicitly express the dependence of $Q$ on $\sigma_y^2$. With this choice for $\sigma_y^2$ and setting $\e^2 = (1+\beta)\sigma_e^2$ applied to \eqref{eq:generalizationbound} we arrive at the statement of Theorem~\ref{thm:generalizationbound}. Note that with this choice of parameters term $\frac1{2\sigma_y^2}(\e^2 - \widehat{L})$ from \eqref{eq:generalizationbound} is equal to $0$.

\section{Omitted proofs}\label{app:proofs}

\begin{lemma}\label{lem:prefsumstonumbers}
Let $\{x_i\}_{i=1}^{k}$ be a set of real numbers. For $i=1, \cdots, k$, define the partial sums $X_i=\sum_{j=1}^{i} x_j$. Then
\begin{align*}
\sum_{i=1}^{k} X_i^2 \geq \frac18 \sum_{i=1}^{k} x_i^2.
\end{align*}
\end{lemma}
\begin{proof}
Define $X_0=0$. Note that for $i=1, \cdots, k$, $X_i = X_{i-1}+x_i$. Hence if $|X_{i-1}| \leq \frac12 |x_i|$ then $|X_{i}|\geq\frac12 |x_i|$ so that $X_{i}^2\geq\frac14 x_i^2$. And if $|X_{i-1}| \geq \frac12 |x_i|$ then $X_{i-1}^2\geq \frac14 x_i^2$. Therefore, $X_{i-1}^2+X_{i}^2 \geq \frac14 x_i^2$. Summing the last inequality over $i=1, \cdots, k$, and adding $X_k^2$ to the left hand side we get $2 \sum_{i=1}^{k} X_i^2 \geq  \frac14 \sum_{i=1}^{k} x_i^2$.
\end{proof}

\begin{lemma}\label{lem:l2lwrbnd}
Let $f(x)= \sum_{i=1}^{k} w_i [x-b_i]_+$, where $0 \leq b_1 \leq \cdots \leq b_k \leq 1 = b_{k+1}$. For $i=1, \cdots, k$, define the partial sums $W_i=\sum_{j=1}^{i} w_j$. Then 
\begin{align*}
\|f\|^2 \geq \frac{1}{12} \sum_{i=1}^{k} W_i^2 (b_{i+1} - b_i)^3.
\end{align*}
\end{lemma}
\begin{proof}
Note that there are $k$ non-overlapping intervals, namely [$b_1, b_2], \cdots, [b_k, 1]$, where the function is potentially non-zero. On the $i$-th interval the function is linear (or more precisely, affine) with a slope of $W_i$ and, by assumption, the interval has length $b_{i+1}-b_i$. On this interval the integral of $f(x)^2$ must have a value of at least $\frac1{12} W_i^2 (b_{i+1}-b_i)^3$. The last statement follows by minimizing the integral of the square of an affine function with slope $W_i$ over the choice of the parameters.
\end{proof}

\begin{lemma}\label{lem:hardone}
Let $f_\theta(x)= \sum_{i=1}^{k} w_i [x-b_i]_+$, where $0 \leq b_1 \leq \cdots \leq b_k < +\infty$.
If $\|f_\theta\|^2 < \frac{1}{12(k+1)^5}$ then there exists $\theta^*$ such that $f_{\theta^*} \equiv_{[0,1]} 0$ and
\begin{align*}
\|\theta - \theta^*\|^2 \leq O \left( k^{13/5} \|f_\theta\|^{4/5} \right).
\end{align*}
\end{lemma}

\begin{proof}
Starting with the parameter $\theta$ that defines the function $f_\theta(x)$, we define a process of changing it until the resulting function is equal to the zero function on $[0, 1]$. Most importantly, this process does not change $\theta$ too much compared to the norm of $f_\theta(x)$. 

Note that there are two ways of setting the function to $0$ on a particular interval. Either, we can make the length of the interval to be $0$. This requires to change one of the bias terms by the length of the interval. Or we set the slope of this interval to be $0$ (assuming that the function is already $0$ at the start of the interval. Our approach uses both of those mechanisms. Let $\theta^0 \leftarrow \theta$. The process has two phases. In the first phase we change the bias terms and in the second phase we change the weights. For $x \in [0,1]$, define the partial sums $W(x)=\sum_{j: b_j \leq x} w_j$. 

\paragraph{First phase.} Let 
$S := \{[b_1,b_2], \dots, [b_{k-1},b_k],[b_k,1]\}$ and  $S_b := \{[l,r] \in S : r - l < |W(l)| \}$. Let $\{[l_0,r_0],[l_1,r_1], \dots, [l_i,r_i]\} \subseteq S_b$ be a maximal continuous subset of intervals in $S_b$. That is, for all $j \in [i]$, $r_j = l_{j+1}$ and the intervals ending at $l_0$ and starting at $r_i$ are not in $S_b$. Perform the following: for all $b_j \in [l_0,r_i]$ set $b_j \leftarrow r_i$. We do this operation for all maximal, continuous subsets of $S_b$. This finishes the first phase. Call the resulting vector of parameters $\theta^1$. We bound
\begin{align}
\|\theta^0 - \theta^1\|^2 
&\leq k \left(\sum_{[l,r] \in S_b} (r-l) \right)^2 \nonumber \\
&\leq k^{13/5} \left(\sum_{[l,r] \in S_b} (r-l)^5 \right)^{2/5} && \text{By the Power Mean Inequality} \nonumber \\
&\leq k^{13/5} \left(\sum_{[l,r] \in S_b} (r-l)^3 W(l)^2 \right)^{2/5} && \text{By definition of } S_b \nonumber \\
&\leq k^{13/5} (12 \|f_\theta\|^2)^{2/5} && \text{By Lemma~\ref{lem:l2lwrbnd}} \label{eq:first stagebound}
\end{align}

\paragraph{Second phase.} Observe that $f_{\theta^1}$ has the following properties. For every $x \in [0,1] \setminus \bigcup_{[l,r] \in S_b} [l,r)$ we have $W^1(x) = W^0(x)$. It is enough to make $W(l) = 0$ for all $[l,r]$ such that $[l,r] \in S \setminus S_b$.  Let $i_1 < i_2 < \dots < i_p$ be all $i_j$'s such that $[b_{i_j}, b_{i_j+1}] \in S \setminus S_b$. Applying Lemma~\ref{lem:prefsumstonumbers} to $\{W_{i_1}, W_{i_2} - W_{i_1}, \dots, W_{i_p} - W_{i_{p-1}}\}$ we get that
\begin{equation}\label{eq:makeallslopeszero}
8\sum_{j=1}^p W_{i_j}^2 \geq W_{i_1}^2 + (W_{i_2} - W_{i_1})^2 + \dots (W_{i_p} - W_{i_{p-1}})^2     
\end{equation}
The RHS of \eqref{eq:makeallslopeszero} gives an upper-bound on the $\|\cdot\|^2$ norm distance needed to change $w_i$'s in $\theta^1$ so that all $W_{i_j} = 0$. It is because we can change $w_1, \dots, w_{i_1}$ by at most $W_{i_1}^2$ to make $W_{i_1} = 0$ and so on for $i_2, \dots, i_p$. Call the resulting vector of parameters $\theta^2$. We bound the change in the second phase
\begin{align}
\|\theta^1 - \theta^2\|^2
&\leq 8 \sum_{j=1}^p W_{i_j}^2 && \text{\eqref{eq:makeallslopeszero}} \nonumber \\
&\leq 8k \left(\frac{1}{k} \sum_{j=1}^p |W_{i_j}^5| \right)^{2/5} && \text{Power Mean Inequality} \nonumber \\
&= 8k^{3/5} \left( \sum_{i : [b_{i}, b_{i+1}] \in S \setminus S_b } |W_i^5| \right)^{2/5} && \text{By definition} \nonumber \\
&\leq 8k^{3/5} \left( \sum_{i : [b_{i}, b_{i+1}] \in S \setminus S_b } (b_{i+1} - b_i)^3 |W_i^2| \right)^{2/5} && \text{By definition of } S_b \nonumber \\
&\leq 8 k^{3/5} \left(12 \|f_\theta\|^2 \right)^{2/5} && \text{By Lemma~\ref{lem:l2lwrbnd}} \label{eq:secondstagebound}.
\end{align}
We conclude by
\begin{align*}
\|\theta^0 - \theta^2\|^2 
&\leq 4 \max\left(\|\theta^0 - \theta^1\|^2, \|\theta^1 - \theta^2\|^2 \right) && \text{Triangle inequality} \\
&\leq 96 k^{13/5} \left(\|f_\theta\|^2 \right)^{2/5} && \text{\eqref{eq:first stagebound} and \eqref{eq:secondstagebound}}
\end{align*}

\end{proof}

\begin{lemma}\label{lem:hardwithbias}[\textbf{With} $b^{(2)}$]
Let $R \in \R_+, \theta \in B_R \cap \supp(P)$ be such that $f_\theta(x)= b^{(2)} + \sum_{i=1}^{k} w_i [x-b_i]_+$, where $0 \leq b_1 \leq \cdots \leq b_k < +\infty$. If $\|f_\theta\|^2$ is small enough, where the bound depends only on $R$ and $k$, then there exists $\theta^*$ such that $f_{\theta^*} \equiv_{[0,1]} 0$ and
\begin{align*}
\|\theta - \theta^*\|^2 \leq O \left( k^{5} R^{4/5} \|f_\theta\|^{2/5} \right) .
\end{align*}
\end{lemma}

\begin{proof}
Let $\e^2 = \|f_\theta\|^2$. For $x \in \R$, define the partial sums $W(x)=\sum_{j: b_j \leq x} w_j$. 
Consider the following cases:
\paragraph{Case $|b^{(2)}| \leq \e^{1/2}$. } We perform $\theta' \leftarrow \theta, b^{(2)'} \leftarrow 0$. By triangle inequality we can bound
$\|f_{\theta'}||^2 \leq \left( \e + |b^{(2)}| \right)^2 \leq 4 \e.$ We apply Lemma~\ref{lem:hardone} to $\theta'$ to obtain $\theta^*$ such that $f_{\theta^*} \equiv_{[0,1]} 0 $ and $\|\theta' - \theta^*\|^2 \leq O\left(k^{13/5} \|f_{\theta'}||^{4/5}\right) \leq O\left(k^{13/5} \e^{2/5}\right)$. We conclude by noticing
\begin{align*}
\|\theta - \theta^*\|^2 
&\leq \left(\|\theta - \theta'\| + \|\theta' - \theta^*\| \right)^2 && \text{Triangle inequality} \\
&\leq \left(\e^{1/2} + O\left(k^{13/10} \e^{1/5}\right) \right)^2 \\
&\leq O\left(k^{13/5} \|f_\theta\|^{2/5}\right) && \text{As } \e^2 \leq 1.
\end{align*}

\paragraph{Case $|b^{(2)}| > \e^{1/2}$. } Without loss of generality assume that $b^{(2)}>0$. There exists $x_0 \in (0,\e/4),$ such that $f_\theta(x_0) = \frac{b^{(2)}}{2}$, as otherwise 
$$\e^2 \geq \int_0^{\e/4} f_\theta(x)^2 \ dx \geq \int_0^{\e/4} (b^{(2)})^2 / 4 \ dx > \e^2.$$ By the mean value theorem there exists $x_1 \in (0,x_0) \setminus \{b_1, \dots, b_k\}$ such that 
\begin{equation}\label{eq:x1bounds}
f_\theta(x_1) \in [b^{(2)}/2, b^{(2)}] \text{ and } W(x_1) \leq \frac{f_\theta(x_0) - f_\theta(0)}{x_0 - 0} \leq -\frac{4b^{(2)}}{2\e} \leq -2\e^{-1/2}.
\end{equation}
We perform the following transformation
\begin{align*}
&\theta' \leftarrow \theta, \\
&\text{for every } i \text{ such that } b_i < x_1 \text{ do } b'_i \leftarrow b_i - x_1 + \frac{f_\theta(x_1)}{W(x_1)}, \\
&i_0 \leftarrow \argmin_i b_i > x_1, \\
&b'_{i_0} \leftarrow 0.
\end{align*}
Observe that we shifted all $b_i$'s exactly so that $f_{\theta'}(0) = 0$. Note also that $b_{i_0} \leq 4\e$ as otherwise by Lemma~\ref{lem:l2lwrbnd}
$$
\e^2 \geq \int_{x_1}^{b_{i_0}} f_\theta(x)^2 \ dx \geq \frac1{12} W(x_1)^2 (b_{i_0} - x_1)^3 > \frac1{12} 4\e^{-1} (3\e)^3 \geq \e^2.
$$
By \eqref{eq:x1bounds} we can bound 
\begin{equation}\label{eq:shiftbiasesthetabound}
\|\theta - \theta'\|^2 
\leq k \left(-x_1 + \frac{f_\theta(x_1)}{W(x_1)}\right)^2 + 16\e^2 \leq O(k\e^2). 
\end{equation}
$f_\theta$ is $R$-Lipshitz wrt to $b_i$'s in $B_R$ thus the triangle inequality and \eqref{eq:shiftbiasesthetabound} gives 
\begin{equation}\label{eq:aftershiftingnorm}
\|f_{\theta'}\|^2
\leq  (\|f_\theta\| + O(R k^{3/2}\e))^2 
\leq O(R^2 k^5 \e^2). 
\end{equation}
We apply Lemma~\ref{lem:hardone} to $f_{\theta'}$, after we removed all $b'_i < 0$ and set $w'_{i_0} \leftarrow \sum_{j \leq i_0} w_j$. Lemma~\ref{lem:hardone} might require to change $w'_{i_0}$, which we can realize with the same cost by changing $\{w_j : j \leq i_0\}$. Thus Lemma~\ref{lem:hardone} and \eqref{eq:aftershiftingnorm} gives us that there exists $\theta^*$ such that $f_{\theta^*} \equiv_{[0,1]} 0$ and $\|\theta' - \theta^*\|^2 \leq O(k^{13/5} k^2 R^{4/5} \e^{4/5})$. We conclude by using the triangle inequality and \eqref{eq:shiftbiasesthetabound} to get
$
\|\theta - \theta^*\|^2 \leq O\left(k^{23/5} R^{4/5} \|f_\theta\|^{4/5}\right).
$
\end{proof}

\begin{lemma}\label{lem:mostimportant}
Let $R \in \R_+, \theta \in B_R \cap \supp(P)$ be such that $f_\theta(x)= b^{(2)} + \sum_{i=1}^{k} w_i [x-b_i]_+$ and $g(x) = b + \sum_{i=1}^c v_i [x - t_i]_+$, where $c \leq k$, $0 \leq b_1 \leq \cdots \leq b_k < +\infty, 0 < t_1 < \dots < t_c < 1$ and $v_1,\dots,v_c \neq 0$. If $\|g - f_\theta\|^2$ is small enough,  where the bound depends only on $g,R$ and $k$, then there exists $\theta^*$ such that $f_{\theta^*} \equiv_{[0,1]} g$ and
\begin{align*}
\|\theta - \theta^*\|^2 \leq O \left( k^{7} R^{4/5} \|g - f_\theta\|^{2/5} \right) .
\end{align*}
\end{lemma}

\begin{proof}
Consider a model on $c+k \leq 2k$ neurons represented as
\begin{equation}
h_{\boldsymbol{\theta}} := \left(b^{(2)} - b\right) + \sum_{i=1}^{k} w_i [x-b_i]_+ - \sum_{i=1}^c v_i [x - t_i]_+, 
\end{equation}
where, to distinguish it from $\theta$, we denoted by $\btheta$ the set of parameters of $h$. Observe that $\|h\|^2 = \|g - f_\theta\|^2$. By Lemma~\ref{lem:hardwithbias} there exists $\btheta^*$ such that $h_{\btheta^*} \equiv_{[0,1]} 0$ and $\|\btheta - \btheta^*\|^2 \leq O \left( k^{5} R^{4/5} \|g - f_\theta\|^{2/5} \right)$. If $\e$ is small enough then the parameters in $\btheta^*$ corresponding to $v_i$'s are all still all non-zero and the bias terms corresponding to $t_i$'s are still all different. As $h_{\btheta^*} \equiv_{[0,1]} 0$ it implies that for every $i \in [c]$ there is a set of bias terms corresponding to $b_j$'s that are exactly at where $t_i$ was moved. Let $\pi : [c] \rightarrow 2^{[k]}$ be the mapping from $t_i$'s to subsets of $b_i$'s certifying that. 

We define $\theta^*$ such that $f_{\theta^*} \equiv_{[0,1]} g$ as the result of two steps. First, changing $\theta$ as its corresponding parameters were changed in the transition $\btheta \rightarrow \btheta^*$. Second, changing the parameters as $v_i$'s and $t_i$'s are changed in $\btheta^* \rightarrow \btheta$ under the map $\pi$. Observe that $\|\theta - \theta^*\|^2 \leq k^2 \|\btheta - \btheta^*\|^2$. It is because in the second step we move at most $k$ bias terms for every parameter corresponding to $t_i$.     
\end{proof}

Proof of Lemma~\ref{lem:complexitytodimension}

\begin{proof}
Let $R \in \R_+$. Notice that $f_\theta$ is $R^2$-Lipschitz with respect to each of its parameters, when restricted to a ball $B_{R}$. This implies that for all $\e > 0$ 
\begin{equation}\label{eq:inclusion1}
(A_g + B_\e) \cap B_R \subseteq \left\{\theta : \|g - f_\theta\|^2 \leq R^4 \e^2 \right\}.
\end{equation}
On the other hand by Lemma~\ref{lem:mostimportant} we have that for small enough $\e$
\begin{equation}\label{eq:inclusion2}
\left\{\theta : \|g - f_\theta\|^2 \leq \e^2 \right\} \cap B_R \cap \supp(P) \subseteq A_g + B_{O \left( k^{7/2} R^{2/5} \e^{1/5} \right) } \subseteq A_g + B_{C(k,R)\e^{1/5} },    
\end{equation}
for some function $C$.
To finish the proof we bound
\begin{align}
&\chi^\#(g, U([0,1])) \nonumber \\
&= \lim_{\epsilon \rightarrow 0}  \frac{\log[\prob_\theta\{\theta: \|g - f_\theta\|^2 \leq \e^2 \}]}{\log(\e)} \nonumber \\
&\stackrel{(1)}{=} \lim_{R \rightarrow \infty} \lim_{\epsilon \rightarrow 0}  \frac{\log[\prob_\theta\{\theta: \|g - f_\theta\|^2 \leq \e^2, \|\theta\|_2 \leq R \}]}{\log(\e)} \nonumber \\
&\stackrel{(2)}{\geq} \lim_{R \rightarrow 
\infty} \lim_{\e \rightarrow 0} \frac{\log \left(\vol( (A + B_{C(k,R) \e^{1/5}}) \cap B_R \cap \supp(P) ) \max_{\theta \in B_R} P(\theta) \right)}{\log(\e)}   \nonumber \\
&\stackrel{(3)}{=} \lim_{R \rightarrow 
\infty} \lim_{\e \rightarrow 0} \frac{\log \left(\vol( (A + B_{C(k,R) \e^{1/5}}) \cap B_R \cap \supp(P)) \right)}{\log(C(k,R) \e^{1/5})} \cdot \frac{\log(C(k,R) \e^{1/5})}{\log(\e)} \nonumber \\
&\stackrel{(4)}{=} \frac1{5}  \lim_{R \rightarrow 
\infty} \lim_{\e \rightarrow 0} \frac{\log \left(\vol( (A + B_{C(k,R) \e^{1/5}}) \cap B_R \cap \supp(P)) \right)}{\log(C(k,R) \e^{1/5})} \nonumber \\
&= \frac1{5} \codim_{P}(A_g), \nonumber
\end{align}
where in (1) we assumed that the two quantities are equal, in (2) we used \eqref{eq:inclusion2}, in (3) we used  $\lim_{\e \rightarrow 0} \frac{\max_{\theta \in B_R} P(\theta)}{\log(\e)} = 0$ and in (4) we used $\lim_{\e \rightarrow 0} \frac{\log(C(k,R)\e^{1/5})}{\log(\e)} = \frac15$. The second bound reads
\begin{align}
&\chi^\#(g, U([0,1])) \nonumber \\
&= \lim_{R \rightarrow \infty} \lim_{\epsilon \rightarrow 0}  \frac{\log[\prob_\theta\{\theta: \|g - f_\theta\|^2 \leq \e^2, \|\theta\|_2 \leq R \}]}{\log(\e)} \nonumber \\
&\stackrel{(1)}{\leq} \lim_{R \rightarrow 
\infty} \lim_{\e \rightarrow 0} \frac{\log \Big(\vol( (A + B_{R^2 \e}) \cap B_R \cap \supp(P)) \cdot \min_{\theta \in B_R \cap \supp(P)} P(\theta) \Big)}{\log(\e)}   \nonumber \\
&\stackrel{(2)}{=} \lim_{R \rightarrow 
\infty} \lim_{\e \rightarrow 0} \frac{\log(\vol( (A + B_{R^2 \e}) \cap B_R \cap \supp(P)))}{\log( R^2 \e)} \cdot \frac{\log(R^2 \e)}{\log(\e)} \nonumber \\
&= \codim_{P}(A_g), \nonumber
\end{align}
where in (1) we used \eqref{eq:inclusion1} and in (2) we used $\min_{\theta \in B_R \cap \supp(P)} P(\theta) > 0$, which is true because $B_R$ is compact.



\end{proof}

Proof of Lemma~\ref{lem:cchangesco-dimension}

\begin{proof}
Let $\ths$ and $\slopes$ denote the vectors of $t_i$'s, and $v_i$'s respectively. Note that if for $i \in [1,c]$ we define  $b_i^{(1)} := t_i$, $w_i^{(2)} := \frac{v_i}{w_i^{(1)}}$ and $b^{(2)} := b$ then for every $x \in [0,1]$
$$
g(x) = \sum_{i=1}^c w_i^{(2)} \cdot w_i^{(1)} \left[x - b_i^{(1)} \right]_+ + b^{(2)}.
$$
Moreover if the neurons $i \in [c+1,k]$ are inactive on $[0,1]$, that is if $b_i^{(1)} > 1$ for all $i > c$, then $g \equiv_{[0,1]} f_\theta$, i.e. functions $g$ an $f_\theta$ agree on $[0,1]$. If we denote by $\bias_{[p,q]}$ the restrictions of $\bias$ to coordinates $p,\dots,q$, then for $\e < \max(t_1, t_2 -t_1, \dots, t_c, 1 - t_c)$ we can write
\begin{align}
&(A_g + B_\e) \cap B_R \cap \supp(P) \nonumber \\
&\supseteq \left\{ \theta : \|\bias_{[1,c]} - \ths\|^2 \leq \frac{\e^2}{3}, \bias_{[c+1,k]} \in [1,M]^{k-c}, \|\w^{(2)} \w^{(1)} - \slopes\|^2 \leq \frac{\e^2}{3}, (b^{(2)} - b)^2 \leq \frac{\e^2}{3} \right\} \cap B_R. \label{eq:inclusionlwrbnd}
\end{align}
Now we will estimate $\vol(\{\w : \|\w^{(2)} \w^{(1)} - \slopes\|^2 \leq \e^2 \} \cap B_R)$.
If $k=1$ and $R^2 > 5|v_1|$:
\begin{align}
&\vol\left(\left\{w^{(1)},w^{(2)} \in \R : (w^{(2)}w^{(2)} - v)^2 \leq \e^2 \right\} \cap B_R \right) \nonumber \\ 
&\geq 2\int_{|v|^{1/2}}^{2|v|^{1/2}} \frac{2\e}{w^{(1)}} \ dw^{(1)} \nonumber \\
&= 4\e(\log(2|v|^{1/2}) - \log(|v|^{1/2})) = 4\log(2) \e. \label{eq:onedimensioncase}
\end{align}
Bound from \eqref{eq:onedimensioncase} generalizes to higher dimensions. If $R^2 > 5\|\slopes\|^2$ then
\begin{equation}\label{eq:boundonslopes}
\vol(\{\w : \|\w^{(2)} \w^{(1)} - \slopes\|^2 \leq \e^2 \} \cap B_R) \geq \kappa \e^c,
\end{equation}
where $\kappa$ is independent of $\e$, $\kappa$ depends only on the volume of balls in $\R^c$ and the constants $4\log(2)$ from \eqref{eq:onedimensioncase}. Now we can lower-bound the co-dimension
\begin{align*}
\codim_P(A_g) &=
\lim_{R \rightarrow 
\infty} \lim_{\e \rightarrow 0} \frac{\log(\vol( (A_g + B_\e) \cap B_R \cap \supp(P)))}{\log(\e)} \\
&\leq \lim_{\e \rightarrow 0} \frac{\log \left(\kappa' (\frac{\e}{\sqrt{3}})^c \cdot (M-1)^{k-c} \cdot \kappa (\frac{\e}{\sqrt{3}})^c \cdot \frac{2\e}{\sqrt{3}} \right)}{\log(\e)} && \text{By \eqref{eq:boundonslopes} and \eqref{eq:inclusionlwrbnd}} \\
&= 2c+1,
\end{align*}
where similarly as before $\kappa'$ is a constant independent of $\e$.

Now we will show an inequality in the other direction. Assume towards contradiction that $\codim(A_g) < 2c+1$. This means that there exists $\theta \in \text{int}(\supp(P)), f_\theta = g$ and $u_1, \dots, u_{3k+1-2c} \in \R^{3k+1}$ linearly independent such that $\theta + \text{ConvHull}(u_1, \dots, u_{3k+1-2c}) \subseteq A_g$. Fix one such $\theta$.

Next observe that 
\begin{equation}\label{eq:globalbiasagrees}
b^{(2)} = b.
\end{equation}
Moreover
\begin{equation}\label{eq:allthresholds}
\{t_1, \dots, t_c\} \subseteq \{\bias_1^{(1)},\dots, \bias_k^{(1)} \},
\end{equation}
because if there was $t_i \not\in \{\bias_1^{(1)},\dots, \bias_k^{(1)} \}$ then $f''_\theta(t_i) = 0$ but $g''(t_i) = v_i \neq 0$. For every $i \in [1,k]$ define $S_i := \{j \in [1,k] : \bias_j^{(1)} = \bias_i^{(1)} \}$. Note that for every $i \in [1,k]$ such that $\bias_i^{(1)} = t_j$ for some $j \in [1,c]$ we have:
\begin{equation}\label{eq:slopesthesame}
\sum_{p \in S_i} \w_p^{(2)} \cdot \w_p^{(1)} = \begin{cases} 
            v_j \mbox{,} & \mbox{if } \bias_i^{(1)} = t_j \\ 
            0 \mbox{,} & \mbox{if } \bias_i^{(1)} \in [0,1] \setminus \{t_1, \dots, t_c\}
        \end{cases} 
\end{equation}
If not then let $i_0$ be such that $\bias_{i_0}^{(1)}$ is the minimal one such that \eqref{eq:slopesthesame} doesn't hold. Note that then $g \not\equiv_{\left[\bias_{i_0}^{(1)}, \bias_{i_0}^{(1)} + \delta \right]} f_\theta$, where $\delta > 0$ is small enough so that $\left\{\bias_1^{(1)},\dots, \bias_k^{(1)} \right\} \cap (\bias_{i_0}^{(1)} , \bias_{i_0}^{(1)} + \delta) = \emptyset$. 
Now observe that \eqref{eq:globalbiasagrees}, \eqref{eq:allthresholds} and \eqref{eq:slopesthesame} give us locally at least $2c+1$ linearly independent equations around $\theta$ which contradicts with $\theta + \text{ConvHull}(u_1, \dots, u_{3k+1-2c}) \subseteq A_g$. Thus $\codim(A_g) \geq 2c+1$. 
\end{proof}

Next we give a helpful fact.

\begin{fact}\label{fact:productdensity}
Let $X,Y$ be two independent random variables distributed according to $\mathcal{N}(0,\sigma_w^2)$. Then for every $a_0 \in \R$ we have that the density of $XY$ at $a_0$ is equal to
\begin{equation}\label{eq:smallfact}
f_{XY}(a_0) = \frac{1}{2\pi \sigma_w^2} \int_{-\infty}^{+\infty} e^{-\frac{1}{2\sigma_w^2}\left(w^2 + \frac{a_0^2}{w^2}\right)} dw = \frac{1}{\sqrt{2\pi \sigma_w^2}}e^{-\frac{|a_0|}{\sigma_w^2}}.
\end{equation}
\end{fact}

Proof of Lemma~\ref{lem:complexityforonechange}

\begin{proof}
To prove the lemma we estimate the probability of $f_\theta$'s close to $g_1$. Without loss of generality assume that $a > 0$.

\paragraph{Upper bound.} We can represent $g_1$ with a single node $i$ by assigning $\sqrt{a}$ to the outgoing weight ($\w^{(2)}_i$), $\sqrt{a}$ to the incoming weight ($\w^{(1)}_i$) of this node, the bias term ($\bias^{(1)}_i$) to $t$ and $b^{(2)}$ to $b$. The bias terms of all other nodes lie in $(1,M]$, i.e. they are inactive in the interval $[0,1]$.

These are exact representations of the function but to compute a lower bound on the probability we should also consider functions that are close to $g_1$. We can change $\w_i^{(1)}, \w_i^{(2)}, \bias_i^{(1)}$ by a little bit and still have a function that satisfies $\|g_1 - f_{\theta}\|^2 \leq \e^2$. We claim that the target probability is lower bounded by
\begin{equation}\label{eq:probabilitylowerbound}
 \left( \frac{\e}{2} \frac{1}{\sqrt{2\pi\sigma_w^2}} e^{-\frac{10 a}{9 \sigma_w^2}} \right) \cdot \left(\frac{9\e} {20 M a} \right) \cdot \left( \frac{\e}{40} \frac{1}{\sqrt{2\pi\sigma_b^2}} e^{-\frac{(|b| + \frac{\e}{40})^2}{2\sigma_b^2}} \right) \cdot \left( \frac{M-1}{M} \right)^{k-1}.
\end{equation}
We arrive at this expression by noting the following facts. By \eqref{eq:smallfact} and the assumption that $a \geq 20\e$ the probability that $\w_i^{(2)} \w_i^{(1)} = a \pm \frac{\e}{2}$ is lower bounded by $\frac{\e}{2} \frac{1}{\sqrt{2\pi\sigma_w^2}} e^{-\frac{10 a}{9 \sigma_w^2}}$. The probability that $\bias_i^{(1)} = t \pm \frac{9\e}{20a}$ is equal $\frac{9\e} {20M a}$. The probability that $b^{(2)} = b \pm \frac{\e}{40}$ is lower bounded by $\frac{\e}{40} \frac{1}{\sqrt{2\pi\sigma_b^2}} e^{-\frac{(|b| + \frac{\e}{40})^2}{2\sigma_b^2}}$. The last term is the probability that all other nodes have bias terms in $[1,M]$. Their weights can realm over the whole space and these nodes don't affect the function on $[0,1]$. We claim that all functions of this form satisfy $\|g_1 - f_{\theta}\|^2 \leq \e^2$. We bound the pointwise difference of $g_1$ and $f_\theta$ in $[0,1]$, i.e. for every $x \in [0,1]$
\begin{align*}
&f_\theta(x) = b + \left(\w_i^{(2)}\w_i^{(1)} \pm \frac{\e}{2}\right)\left[x - \left(\bias_i^{(1)} \pm \frac{9\e}{20a}\right)\right]_+ \pm \frac{\e}{40} \\
&= b + \w_i^{(2)}\w_i^{(1)}\left[x - \left(\bias_i^{(1)} \pm \frac{9\e}{20a}\right)\right]_+ \pm \frac{\e}{2}\left[x - \left(\bias_i^{(1)} \pm \frac{9\e}{20a}\right)\right]_+ \pm \frac{\e}{40} \\
&= b + \w_i^{(2)}\w_i^{(1)}\left[x - \bias_i^{(1)}\right]_+ \pm \w_i^{(2)}\w_i^{(1)} \frac{9\e}{20a} \pm \frac{\e}{2}\left(1 + \frac{9\e}{20a}\right) \pm \frac{\e}{40} \\
&= b + \w_i^{(2)}\w_i^{(1)}\left[x - \bias_i^{(1)}\right]_+ \pm \frac{9\e}{20} \pm \frac{\e}{2}\left(\frac{21}{20} + \frac{9\e}{20a}\right) && \text{As } \w_i^{(2)}\w_i^{(1)} = a \\
&= b + \w_i^{(2)}\w_i^{(1)}\left[x - \bias_i^{(1)}\right]_+ \pm \e && \text{As } a \geq 20\e, \\
\end{align*}
which implies that for such representations $\|g_1 - f_{\theta}\|^2 \leq \e^2$. From \eqref{eq:probabilitylowerbound} we get an upper bound on the sharp complexity 
\begin{align}
&\chi^\#( g_1, \e^2) \nonumber \\
&\leq -\log\left[ \left( \frac{\e}{2} \frac{1}{\sqrt{2\pi\sigma_w^2}} e^{-\frac{10 a}{9 \sigma_w^2}} \right) \cdot \left(\frac{9\e} {20 M a} \right) \cdot \left( \frac{\e}{40} \frac{1}{\sqrt{2\pi\sigma_b^2}} e^{-\frac{(|b| + \frac{\e}{40})^2}{2\sigma_b^2}} \right) \cdot \left( \frac{M-1}{M} \right)^{k-1}\right] \nonumber \\
&\leq \frac{10}{9} \left(\frac{a}{\sigma_w^2} + \frac{|b|}{\sigma_b^2} \right) + \log\left(M a \right) - (k-1) \log(1 - 1/M) + \log(2\pi \sigma_w \sigma_b) + 7 - 3 \log \left(\e \right)\nonumber  \\
&\leq \frac{10}{9} \left(\frac{a}{\sigma_w^2} + \frac{|b|}{\sigma_b^2} \right) + \log\left(M a \right) - (k-1) \log(1 - 1/M) + 10 - 3 \log \left(\e \right). &&\text{As } \sigma_b^2 \leq \frac{1}{\sigma_w^2} \nonumber \\
&\leq 2\left(\frac{a}{\sigma_w^2} + \frac{|b|}{\sigma_b^2} \right) + 11 - 3\log(\e),\label{eq:uprbndonechange} 
\end{align}
where in the last inequality we used that $\log(x) < x/2, \ \log(1+x) < x$ for $ x> 0$ and the assumption $k \leq M \leq \frac{1}{\sigma_w^2} $.

\begin{note}
Observe that according to Corollary~\ref{lem:cchangeslimit} we have that $\chi^\#(g_1, \dist_x) \leq 3$. Recall that $\chi^\#(g_1, \dist_x) = \lim_{\e \rightarrow 0} -\chi^\#(g_1, \dist_x, \e^2)/\log(\e)$. This means that, at least approximately, if we took the bound from \eqref{eq:uprbndonechange}, divided it by $-\log(\e)$ we would get an upper bound on  $\chi^\#(g_1, \dist_x)$. This would yield for us $\chi^\#(g_1, \dist_x) \leq 3$, as all other terms go to $0$ when $\e \rightarrow 0$. 
\end{note}

\paragraph{Lower bound.} There are other $\theta$'s that represent the function approximately. For example we could represent $g_1$ with more than $1$ node, by \say{spreading} the change of slope $a$ over many nodes. Another possibility is that a number of nodes with the same bias terms $t \neq b \in [0,1]$ effectively cancel out. These $\theta$'s contribute to the probability and decrease the complexity. 



Let $\theta$ be such that $\|g_1 - f_\theta\|^2 \leq \e^2$ and let $S := \{i \in \{1, \dots, k\} : \bias_i^{(1)} \in [t - 9\e^{1/2}, t + 9\e^{1/2}] \}$. 
Assume towards contradiction that $\sum_{i \in S} |\w_i^{(1)} \w_i^{(2)}| < a -\e^{1/4}$. This implies that either
\begin{equation}\label{eq:case1}
\sum_{i : \bias_i^{(1)} \in [t - 9\e^{1/2},t]} |\w_i^{(1)}\w_i^{(2)}| < f'_\theta(t) - \e^{1/4}/2
\end{equation}
or
\begin{equation}\label{eq:case2}
\sum_{i : \bias_i^{(1)} \in [t, t + 9\e^{1/2}]} |\w_i^{(1)}\w_i^{(2)}| < a - f'_\theta(t) -  \e^{1/4}/2.
\end{equation}
Assume that \eqref{eq:case2} holds. A similar argument covers \eqref{eq:case1}. Now consider two cases.

\paragraph{Case 1.} For all $x \in [t, t+ 3\e^{1/2}]$ we have $f_\theta(x) > a(x-t) + \e^{3/4}$. Then $\|g_1 - f_{\theta}\|^2 \geq 3\e^{1/2} \cdot \e^{3/2} > \e^2$ $\lightning$. 
\paragraph{Case 2.} There exists $x_0 \in [t,t+3\e^{1/2}]$ such that \begin{equation}\label{eq:x0property}
f_\theta(x_0) < a(x_0-t) + \e^{3/4}.
\end{equation}
By \eqref{eq:case2} we know that for all $x \in [t,t+9\e^{1/2}]$ we have $f'_\theta(x) < a - \e^{1/4}/2$. This means that $f_\theta(x)$ is below a linear function of slope $a-\e^{1/4}/2$ passing through $(x_0,f_\theta(x_0))$. Now we lower bound the error using the fact that $f_\theta$ is below this line. 
\begin{align}
&\|g_1 - f_{\theta}\|^2 \nonumber \\
&\geq \int_{x_0}^{t+9\e^{1/2}} \left[a(x-t) - \left(f(x_0) + \left(a - \e^{1/4}/2\right)(x-x_0)\right) \right]^2 \mathbbm{1}_{\{a(x-t) > f(x_0) + \left(a - \e^{1/4}/2\right)(x-x_0)\}} dx \label{eq:errolwrbnd}
\end{align}
Note that the function $\delta(x) := a(x-t) - \left(f(x_0) + \left(a - \e^{1/4}/2\right)(x-x_0)\right)$ is increasing in $x$ and moreover 
\begin{align}
\delta(7\e^{1/2} + t) 
&=  a(x_0-t) - f(x_0) + \frac{\e^{1/4}}{2}(7\e^{1/2} + t - x_0) \nonumber \\
&\geq -\e^{3/4} + 2\e^{3/4} && \text{By \eqref{eq:x0property} and } x_0 < t + 3\e^{1/2} \nonumber \\
&\geq \e^{3/4}.\label{eq:fnctvalue}
\end{align}
Combining \eqref{eq:errolwrbnd} and \eqref{eq:fnctvalue} we get that
$$
\|g_1 - f_{\theta}\|^2 \geq 2\e^{1/2} \cdot \e^{6/4} > \e^2,
$$
which is a contradiction $\lightning$.

We arrived at a contradiction in both cases thus $\sum_{i \in S} |\w_i^{(1)} \w_i^{(2)}| \geq a -\e^{1/4}$. We claim that for every such $S$ the probability of $\sum_{i \in S} |\w_i^{(1)} \w_i^{(2)}| \geq a -\e^{1/4}$ is at most
\begin{equation}\label{eq:forafixedset}
\left( \frac{18\e^{1/2}}{M} \right)^{|S|} \int_{a - \e^{1/4}}^\infty x^{(|S| - 1)} \frac{2^{|S|}}{|S|!} \cdot \frac{2}{\sqrt{2\pi \sigma_w^2}}e^{-\frac{x}{\sigma_w^2}} \ dx.
\end{equation}

We arrive at this expression by noting that $x^{(|S| - 1)} \frac{2^{|S|}}{|S|!}$ is the area of an $\ell_1$ sphere of radius $x$ in $|S|$ dimensions; the density for $\w_i$'s satisfying $\sum_{i \in S} |\w_i^{(1)} \w_i^{(2)}| = x$ is, by Fact~\ref{fact:productdensity}, $\frac{2}{\sqrt{2\pi \sigma_w^2}}e^{-\frac{x}{\sigma_w^2}}$; the probability that a single bias term is equal to $t \pm 9\e^{1/2}$ is $ \frac{18\e^{1/2}}{M}$.

Now we upper bound the probability of all these functions by taking a union bound over sets $S$. We get 
\begin{align}
&\prob_\theta \left[\|g_1 - f_{\theta}\|^2 \leq \e^2 \right] \nonumber \\
&\leq \sum_{S \subseteq \{1,\dots,n\}} \int_{a - \e^{1/4}}^\infty x^{(|S| - 1)} \frac{2^{|S|}}{|S|!} \cdot \left( \frac{18\e^{1/2}}{k} \right)^{|S|} \cdot \sqrt{\frac{2}{\pi \sigma_w^2}}e^{-\frac{x}{\sigma_w^2}} \ dx && \text{By \eqref{eq:forafixedset} and $k \leq M$} \nonumber \\
&\leq \sqrt{\frac{2}{\pi \sigma_w^2}} \sum_{i=1}^k \int_{a - \e^{1/4}}^\infty \binom{k}{i} \frac{2^{i}}{i!} \left( \frac{1}{2k} \right)^{i} \cdot x^{i - 1} e^{-\frac{x}{\sigma_w^2}} \ dx  && \text{As } 18\e^{1/2} \leq \frac12 \nonumber \\
&\leq \sqrt{\frac{2}{\pi \sigma_w^2}} \sum_{i=1}^k \int_{a/2}^\infty \frac{x^{i-1}}{2^{i-1}} e^{- \frac{x}{\sigma_w^2}} \ dx && \text{As } \binom{k}{i} \leq k^i, i! \geq 2^{i-1}, a \geq 2\e^{1/4} \label{eq:probuprbnd}
\end{align}
For every $i \in [1,k]$ we can upper bound 
\begin{align}
&\int_{a/2}^\infty \frac{x^{i-1}}{2^{i-1}} e^{-\frac{x}{\sigma_w^2}} \ dx \nonumber \\
&\leq \int_{a/2}^2 e^{-\frac{x}{\sigma_w^2}} \ dx + \left[-\frac{x^i}{2^{i-1}} e^{-\frac{x}{\sigma_w^2}} \right]_2^\infty &&\text{As } \left(-x^i e^{-\frac{x}{\sigma_w^2}}\right)' \geq x^{i-1} e^{-\frac{x}{\sigma_w^2}} \text{ for } x \geq 2 \nonumber \\
&\leq \left[-\sigma_w^2 e^{-\frac{x}{\sigma_w^2}} \right]^2_{a/2}  + 2e^{-\frac{2}{\sigma_w^2}} \nonumber \\
&\leq \sigma_w^2 e^{-\frac{a}{2\sigma_w^2}} + 2e^{-\frac{2}{\sigma_w^2}} \nonumber \\
&\leq 3e^{-\frac{a}{2\sigma_w^2}} && \text{As } \sigma_w^2 \leq 1, a \leq 2 \label{eq:integralbnd}
\end{align}
Plugging \eqref{eq:integralbnd} back to \eqref{eq:probuprbnd} we get 
\begin{equation}\label{eq:probupr}
\prob_\theta \left[\|g_1 - f_{\theta}\|^2 \leq \e^2 \right] \leq \sqrt{\frac{18}{\pi}}\frac{k}{\sigma_w} e^{-\frac{a}{2\sigma_w^2}}.  
\end{equation}
With \eqref{eq:probupr} we can bound the sharp complexity
\begin{align*}
\chi^\#(\dist_x, g_1, \e^2) 
&\geq \frac{a}{2\sigma_w^2} - \log(k/\sigma_w) + \log(\sqrt{2\pi}) \\
&\geq \frac{a}{3\sigma_w^2} && \text{As } \Omega(\sigma_w^2\log(k/\sigma_w)) \leq |a|.
\end{align*}

\end{proof}

\end{document}


%

%

\onecolumn
\aistatstitle{Supplementary Material}

\section{$\e$-Complexity}

\paragraph{Variational Complexity}
Consider the case $d=1$, and let $f : [0, 1] \xrightarrow{} \R$ be continuous and piece-wise linear. I.e., there is a sequence of points $0=x_1 < x_2 < \cdots < x_{k+1}=1$ so that for $x \in [x_i, x_{i+1}]$, $1 \leq i < k+1$,
\begin{align} \label{equ:polygone}
f(x) = c_i + \alpha_i (x-x_i),
\end{align}
for some constants $c_i$ and $\alpha_i$, where $c_1=c$ and $c_{i+1} = c_i + \alpha_i (x_{i+1}-x_i)$. Then $f$ can be written as a sum of ReLU functions, 
\begin{align} \label{equ:firstrepresentation}
f(x) = c + \sum_{i=1}^{k} a_i [x-x_i]_+,
\end{align}
where $a_1=\alpha_1$ and $a_{i}=\alpha_{i}-\alpha_{i-1}$, $i=2, \cdots, k$. 
Let us now introduce a
\say{complexity} measure for a function $f_{\theta}$ implemented by the NN with $k$ hidden nodes via the parameters $\theta$. We start by introducing a complexity measure for a particular choice of the network parameters. The complexity of the function will then be the minimum complexity of the network that represents this function. We choose
\begin{align} \label{equ:complexitymeasure}
C_k(\theta) = \frac12 \| \theta_w\|^2 = \frac12\left( \|W^{(1)}\|_F^2 + \|\w^{(2)}\|_2^2 \right),
\end{align}
i.e., it is half the squared Euclidean norm of the {\em weight parameters}. 

If we use the representation (\ref{equ:firstrepresentation}) in its natural form, i.e.,  $w^{(2)}_i =a_i$ and $W^{(1)}_i = 1$, then we have $C_k(\theta) = \frac12 \sum_{i=1}^{k} (a_i^2+1)$. But we can do better. Write
\begin{align} \label{equ:secondrepresentation}
f(x) = c + \sum_{i=1}^{k} w^{(2)}_i [W^{(1)}_i(x-x_i)]_+,
\end{align}
where $w^{(2)}_i =a_i/\sqrt{|a_i|}$ and $W^{(1)}_i = |w^{(2)}_i |$. This gives us a complexity measure $C_k(\theta) = \sum_{i=1}^{k} |a_i| = \sum_{i=1}^{k} |\alpha_i-\alpha_{i-1}|$, where $\alpha_0=0$. Indeed, it is not very hard to see, and it is proved in \citet{srebronormbound}, that this is the best one can do even if we keep $f(x)$ fixed and are allowed to let the number $k$ of hidden nodes tend to infinity. In other words, for the function $f$ described in (\ref{equ:polygone}) we have
\begin{align*}
C(f) = \inf_{k \in \N, \theta: f_\theta = f} C_k(\theta) = \variation(f'),
\end{align*}
where $\variation(f')$ denotes the total variation of $f'$, the derivative of $f$. Why total variation?
Note that $\alpha_i$ denotes the derivative of the function so that $|\alpha_i-\alpha_{i-1}|$ is the change in the derivative at the point $x_i$. Therefore, $\sum_{i=1}^{k} |\alpha_i-\alpha_{i-1}|$ is the total variation associated to this derivative. 


If we consider a general function $f: [0, 1] \xrightarrow{} \R$ then for every $\epsilon>0$, $f$ can be uniformly approximated by a piecewise linear function, see \citet{shekhtman82}. As $\epsilon$ tends to $0$ for the \say{best} approximation the variation of the piece-wise linear function converges to the total variation of $f'$. This can equivalently be written as the integral of 
$|f''|$.
It is therefore not surprising that if we look at general functions $f: \R \xrightarrow{} \R$ and let the network width tend to infinity then the lowest cost representation has a complexity of
\begin{align} \label{equ:complexity}
C(f) = \max \left(\int |f''(x)| dx, |f'(-\infty) + f'(+\infty)| \right).
\end{align}
As we previously mentioned, this concept of the complexity of a function was introduced in \citet{srebronormbound} and this paper also contains a rigorous proof of (\ref{equ:complexity}). (Note: The second term in \eqref{equ:complexity} is needed
when we go from a function that is supported on a finite domain to $\R$. To see this consider the complexity of $f(x) = \alpha x$. It is equal to $2\alpha$ ($f(x) = \sqrt{\alpha} [\sqrt{\alpha} x]_+ - \sqrt{\alpha} [-\sqrt{\alpha} x]_+$) but $\int |f''(x)| dx = 0$.)

\paragraph{Sharp versus Variational Complexity.} Now we explain how the notion of sharp complexity is, in some regimes, equivalent to the variational complexity. This gives a concrete example of our promise that sharp complexity aligns well with natural complexity measures.

Recall that we consider a NN with $k$ nodes in the intermediate layer, a scalar input and a scalar output. We therefore have $ 2k$ weights, namely $k$ weights from the scalar {\em input} to the intermediate nodes (which we henceforth call {\em input} weights) and $k$ weights from the intermediate nodes to the {\em output} (which in the sequel we will call {\em output} weights). Let us start by looking at a simplified model where we only consider the $k$ output weights. We still assume that the incoming weights behave as they should (as we will see shortly, they will take on values equal to the outgoing weights) but we ignore them in the probabilistic expression.
We discuss the real model at the end. Not much will change.

Assume at first that the function we want to represent is of the form (\ref{equ:secondrepresentation}) and requires only a single change of the derivative. I.e., the piece-wise linear function consists of two pieces and we require only one term in the sum. Call this function $g_1$,\comm{GG}{Change $f_1$ to $g_1$} where the $1$ indicates that there is only a single change of the derivative. Let us assume that this change is of magnitude $a$. Further, assume that the prior $\alpha(\theta)$ is Gaussian  where the components are independent. The vector of means is $\mu_{\theta} = (\mu_{\theta_w}=(0,\cdots, 0), \mu_{\theta_b})$ and the vector of variances is $\sigma^2_{\theta} = (\sigma^2_{\theta_w}=(\sigma^2_w,\cdots, \sigma^2_w), \sigma^2_{\theta_b}=(0,\cdots, 0))$. Then $\alpha(\theta) \propto e^{-\frac{1}{2 \sigma_w^2} \| \theta_w\|^2}$.

We now ask what is the value of $\chi^\#(\dist_\xv, g_1, \e)$- as this is what appears in our bound from Section~\ref{sec:intro}. We claim that for small $\e$
\[
-\log \left( \prob[\expectation_{x \sim \dist_x} [(g_1(x) - f_{\theta}(x))^2] \leq \e] \right) \approx \frac{a}{2\sigma_w^2}.
\]
To show that we estimate the probability of $f_\theta$'s close to $g_1$.

If we represent this change of derivative by a single node then we know that we can assign $\sqrt{a}$ to the outgoing weight of this node and $\sqrt{a}$ to the incoming weight of this node. This corresponds to a network where all outgoing weights are zero except one that is $\sqrt{a}$ (and the same for the incoming weights). But we can also take the weight $\sqrt{a}$ and \say{spread it out} over the $k$ nodes in any fashion we want, assuming only that the sum of squares of the weights equals $a$. In fact, take a vector $\vv$ of length $k$ of Euclidean norm squared equal to $1$. Let the vector have components $\vv_i$ and assign the value $\sqrt{a} \vv_i$ to the outgoing weight of node  $i$ (and of course pick the bias terms accordingly and also the incoming weights). Then together all these $k$ nodes will represent the same change in the derivative and the complexity will also be the same. In principle all the components $\vv_i$ should be non-negative so that we are working in the positive orthant. In order to avoid this complication let us assume that we have the following further over-parametrization and that along each connection we not only have the weight but we have in addition the possibility to multiply by an element of $\{\pm1\}$ \comm{GG}{Does this plus minus change anything}. In this way we have no restriction on the sign of the weight. Of course, the two models have equal expressive power. For the rest of this section we will assume this model.

In words, if we limit ourselves to the outgoing weights then we have a $k$-dimensional sphere of radius $\sqrt{a}$, where every single point on this sphere represents $f_1$ exactly. Further, if we do not insist that $f_1$ is represented exactly but allow a small deviation then we can extend the sphere to a spherical shell. Every point in this shell gives an approximate representation of $f_1$, each having roughly the same $\|\theta_w\|^2$. Figure~\ref{fig:sphericalshell} depicts this situation for the case $k=3$. In order to show the spherical shell the sphere is cut open.
\begin{figure}[tb]
\begin{center} \includegraphics[width=8cm]{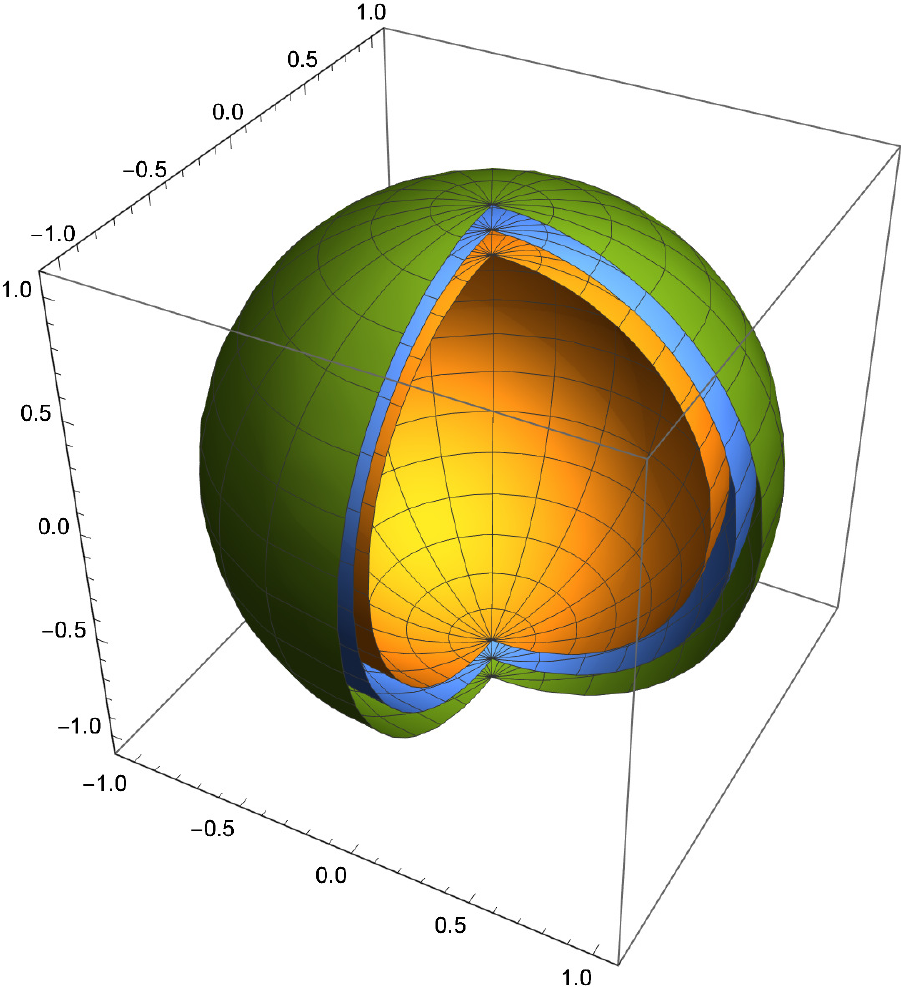}
\end{center}
\caption{Spherical shell of parameters that approximately represent $f_1$. The blue surface shows the sphere of parameters that give exact representations. The green and the orange surface indicate the outer and the inner boundary of the spherical shell giving approximate representations.}
\label{fig:sphericalshell} 
\end{figure}

What is the volume of all these vectors $\theta$? (Recall, to simplify our current presentation we limit $\theta$ to the outgoing weights.) As discussed, the exact representations lie on the surface of a $k$-dimensional sphere of radius $\sqrt{a}$. The area of this sphere is equal to $
    a^{\frac{k-1}{2}} \frac{2 \pi^{k/2}}{\Gamma(k/2)}$.
Thus the probability is \comm{GG}{Proportional versus equal} proportional to
\begin{align}
    a^{\frac{k-1}{2}} \frac{2 \pi^{k/2}}{\Gamma(k/2)} \cdot \epsilon'  \cdot e^{ - \frac{1}{ 2\sigma_w^2} a}. \label{eq:probabilityofg1}
\end{align}
\comm{GG}{Change all $A(f,S)$ to $N L_\samp(f)$}
We arrive at this expression by starting with the surface area, multiplying this by the \say{thickness} of the spherical shell, which we call $\epsilon'$ and finally multiplying with the \say{height} of the density inside the shell, which is proportional to $e^{-\frac{1}{2\sigma_w^2}a}$ as discussed. 

Let us write this in the form 

\begin{align*}
    2 \sqrt{2} (\sigma_w^2\pi)^{k/2} \underbrace{\left(\frac{\sqrt{2a}^{k-1} e^{-\frac{\sqrt{2 a}^2}{ 2\sigma_w^2}}}{\sigma_w^2
    (2 \sigma_w^2)^{k/2-1} \Gamma(k/2)} \right)}_{(*)} \cdot \epsilon.
\end{align*}
\comm{GG}{Theres a weird sigmaw inside the denominator}
Recall that the density of a chi-distribution with $k$ degrees of freedom, i.e.,
the distribution of the square root of the sum of the squares of $k$ independent Gaussians, has the form
\begin{align*}
\frac{x^{k-1} e^{-x^2/2}}{2^{k/2-1} \Gamma(k/2)}.
\end{align*}
Further recall that this density is unimodal for $k>1$, i.e., it has a unique maximum and it first increases up to this maximum and then decreases back to zero thereafter. The maximum is at $x=\sqrt{k-1}$.


We conclude that $(*)$ corresponds to a scaled chi-distribution with $k$ degrees of freedom where instead of zero-mean unit-variance Gaussians we pick zero-mean Gaussians with variance $\sigma_w^2$. Note that such a scaled chi-distribution takes on its maximum value at $\sqrt{k-1} \sigma_w$. The standard scaling of the variance of the weights is $\sigma_w^2 \propto 1/k$. For simplicity let's assume that $\sigma_w^2 = 1/(k-1)$ \comm{GG}{is it ok?}. Then the maximum value is attained at $1$.


Taking $-\log$ of \eqref{eq:probabilityofg1} we get that
\begin{align*}
\chi^\#(\dist_\xv, g_1, \e) 
&\leq \frac{a}{2\sigma_w^2} -\frac{k-1}{2}\log(a) - \log(\frac{2\pi^{k/2}}{\Gamma(k/2)}) - \log(\e') \\
&= \frac{C(g_1)}{2\sigma_w^2} - \poly(k,\log(a),\log(\e')) \\
&= \frac{(k-1)C(g_1)}{4} - \poly(k,\log(a),\log(\e'))
\end{align*}
Why do we have an inequality instead of equality in the first line? The reason is that
there can be other functions that still count towards the probability but are of different form. We will show however that these other functions contribute negligibly to $\chi^\#$.

So consider a second function, call it $f_d$ which in addition to the one change of its derivative also has some other \say{wiggles}, so that the total number of slope changes is $d$. Necessarily its complexity is higher, call it $a'$. Consider the probability of functions close to $f_d$. Recall that $a'>a$. Hence the squared radius of the sum of the squares of the outgoing weights \comm{GG}{Is radius $a$ or $2a$} representing $f_d$ is $a'>a$. Thus the density of $\alpha(\theta)$ is smaller for these functions than for $g_1$. But there is a second, even more important, factor that reduce the probability of $f_d$ compared to $g_1$. Not all the points on the sphere of square radius $a'$ correspond to exact representations of $f_d$. In fact, the points that correspond to exact representations of $f_d$ lie on a lower dimensional sub-manifold, reducing the probability further significantly. We claim that this probability has the form
\begin{align*}
   S(f, k-d) \cdot e^{- \frac{2 a}{ 2 \sigma_w^2}} \cdot \epsilon^{d} \cdot e^{-\frac{1}{ \sigma_y^2} A(f, \samp)},
\end{align*}
where $S(f, k-d)$ denotes the area of the subset of the sphere in $k$ dimensions that represents this function exactly, seen as a $k-d$ dimensional manifold. 

All this is easiest seen by looking at the case $k=3$. As we discussed, for $f_1$ the set of exact representations forms a sphere in $\R^3$ since all three nodes can be used to implement this change but there is one global constraint. This leads to the spherical shell of approximate representations as seen in Figure~\ref{fig:sphericalshell}. Next consider $f_2$, which contains one more change of the derivative. At least one of the three nodes will be needed to implement this second change of the slope, leaving only two nodes to implement the main change of the derivative.
So the exact representations will now lie on a slightly bigger sphere of squared radius $2 a'$ in $\R^3$ and they will correspond to circles of squared radius $2 a$ (each such of the six circles corresponds to picking two of the three nodes and implementing the change of the slope already present in $f_1$ via these two nodes, hence a circle since we can freely distribute the weight between these two nodes). If we \say{expand} one such circle by allowing representations \say{close by} we get a torus.
Hence the approximate implementations will now correspond to tori that are embedded on the surface of a sphere. This is shown in Figure~\ref{fig:tori}. As mentioned, the sphere itself is slightly larger, due to the larger complexity of the function. \comm{GG}{Here we kind of need that chi distribution business} But as we discussed, even if we took the whole spherical shell that corresponds to this slightly larger sphere it would contain less probability. In addition, we only take a small subset of the probability, due to the extra combinatorial constraint. In summary, if we assume that we sample from the posterior distribution, there is a strong regularization term that gives preference to \say{simple} functions, hence avoiding over-fitting.

\begin{figure}[tb]
\begin{center} \includegraphics[width=8cm]{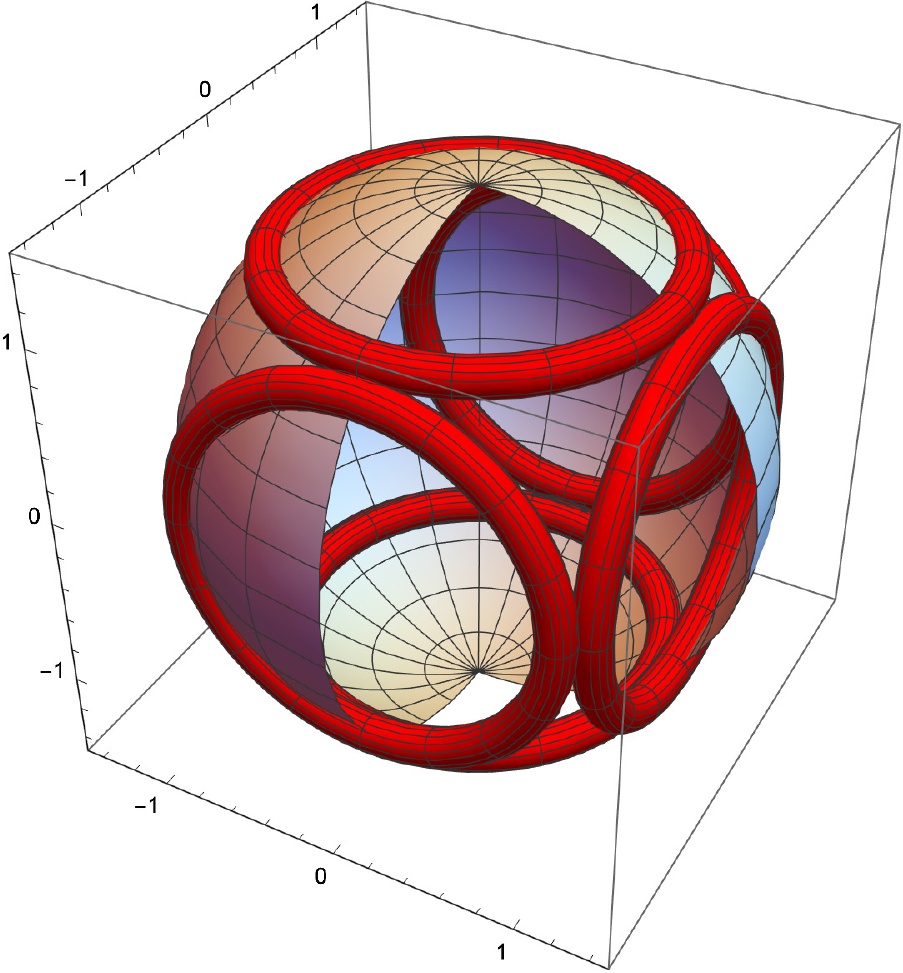}
\end{center}
\caption{The six tori indicate the region of the parameter space (outgoing weights) that approximately represent the function $f_2$, a function with two changes of the derivative. Note that the exact representations correspond to six circles, i.e., it has dimension $1$, whereas for a single change of the derivative it has dimension $2$. If we want to implement a function with three changes of the derivative, then the exact representations are points (dimension $0$).}
\label{fig:tori} 
\end{figure}

So far we considered a simplified model where we only looked at the outgoing weights. But not much changes if we look at the outgoing {\em and} incoming weights.  The manifold that contains the exact representations is now embedded in $2k$-dimensional instead of $k$-dimensional space since
for each node the incoming weight is identical to the outgoing weight. If we apply a suitable unitary transform we see that we can get to essentially the original case but where now all the coordinates that are involved in the manifold are stretched by a factor $\sqrt{2}$, and of course we are in an ambient space of dimension $2k$ instead of $k$. E.g., if we consider the function $f_1$ then 
the associated posterior probability of all it's approximate realization is proportional to  
\begin{align*}
   (2a)^{\frac{k-1}{2}} \frac{2 \pi^{k/2}}{\Gamma(k/2)} \cdot \epsilon^{k+1} \cdot e^{-\frac{1}{ \sigma_y^2} A(f, \samp) - \frac{1}{ \sigma_w^2} a}.
\end{align*}

Let us write this as
\begin{align*}
    2(2\sigma_w^2)^{k/2-1} \pi^{k/2}   \underbrace{\left( \frac{(\sqrt{2 a})^{k-1} e^{-\frac{(2a)}{ 2\sigma_w^2}}}{{
    (2 \sigma_w^2)^{k/2-1}} \Gamma(k/2)} \right)}_{(**)} \cdot 
    \epsilon^{k+1} 
    e^{-\frac{1}{ \sigma_y^2} A(f, \samp)}.
\end{align*}